\documentclass[10pt,twocolumn,letterpaper]{article}

\usepackage[pagenumbers]{cvpr} 

\makeatletter
\@namedef{ver@everyshi.sty}{}
\makeatother

\usepackage{graphicx}
\usepackage{amsmath}
\usepackage{amssymb}
\usepackage{booktabs}

\usepackage{bbm}
\usepackage{bm}

\usepackage{amsmath, amsfonts}  
\usepackage{amsthm}
\usepackage{bm}
\newcommand{\bx}{{\bm x}}
\newcommand{\bt}{{\bm t}}

\newcommand{\bp}{{\bm p}}

\newcommand{\boldb}{{\bm b}}

\newcommand{\bW}{{\bf W}}

\newcommand{\bX}{{\bf X}}

\newcommand{\bphi}{{\bm \phi}}

\newcommand{\bbR}{\mathbb R}
\newcommand{\bbS}{\mathbb S}

\newcommand{\bbE}{\mathbb E}

\newcommand{\cX}{\mathcal X}

\newcommand{\cL}{\mathcal L}

\newcommand{\cD}{\mathcal D}
\newcommand{\cO}{\mathcal O}

\newtheorem{proposition}{Proposition}

\usepackage[pagebackref,breaklinks,colorlinks]{hyperref}

\usepackage{pgfplots}
\usepackage{pgfplotstable}
\usepackage{pgf,tikz}
\pgfplotsset{compat=1.17}

\newcommand{\leg}[1]{\addlegendentry{#1}}

\usepackage[capitalize]{cleveref}
\crefname{section}{Sec.}{Secs.}
\Crefname{section}{Section}{Sections}
\Crefname{table}{Table}{Tables}
\crefname{table}{Tab.}{Tabs.}

\newcommand{\ours}{DiHT\xspace}

\newcommand{\objects}{objects\xspace}
\newcommand{\obj}{obj\xspace}

\def\cvprPaperID{4787} 
\def\confName{CVPR}
\def\confYear{2023}

\usepackage{array}
\usepackage{multirow}
\usepackage{comment}
\usepackage{soul}
\newcolumntype{L}[1]{>{\raggedright\let\newline\\\arraybackslash\hspace{0pt}}m{#1}}
\newcolumntype{C}[1]{>{\centering\let\newline\\\arraybackslash\hspace{0pt}}m{#1}}
\newcolumntype{R}[1]{>{\raggedleft\let\newline\\\arraybackslash\hspace{0pt}}m{#1}}

\definecolor{amber_custom}{rgb}{1.0, 0.49, 0.0}
\definecolor{brass}{rgb}{0.71, 0.65, 0.26}
\definecolor{darkred}{rgb}{0.75, 0.0, 0.0}
\definecolor{darkblue}{rgb}{0.0, 0.0, 0.75}
\definecolor{darkgreen}{rgb}{0.0, 0.50, 0.0}

\usepackage{bbding}

\begin{document}

\title{Filtering, Distillation, and Hard Negatives for Vision-Language Pre-Training}

\author{Filip Radenovic$^1$, Abhimanyu Dubey$^1$$^*$, Abhishek Kadian$^1$$^*$, Todor Mihaylov$^1$$^*$, Simon Vandenhende$^1$$^*$\\Yash Patel$^2$$^\dagger$, Yi Wen$^1$, Vignesh Ramanathan$^1$ and Dhruv Mahajan$^1$$^\ddagger$ \\
$^1$Meta AI \quad $^2$CTU in Prague
}
\maketitle

\def\thefootnote{*}\footnotetext{Equal contribution. $^\dagger$Work done at Meta AI. $^\ddagger$Research Lead.}
\def\thefootnote{\arabic{footnote}}

\begin{abstract}
Vision-language models trained with contrastive learning on large-scale noisy data are becoming increasingly popular for zero-shot recognition problems. In this paper we improve the following three aspects of the contrastive pre-training pipeline: dataset noise, model initialization and the training objective.  First, we propose a straightforward filtering strategy titled Complexity, Action, and Text-spotting (CAT) that significantly reduces dataset size, while achieving improved performance across zero-shot vision-language tasks. Next, we propose an approach titled Concept Distillation to leverage strong unimodal representations for contrastive training that does not increase training complexity while outperforming prior work. Finally, we modify the traditional contrastive alignment objective, and propose an importance-sampling approach to up-sample the importance of hard-negatives without adding additional complexity. On an extensive zero-shot benchmark of 29 tasks, our Distilled and Hard-negative Training (\ours) approach improves on 20 tasks compared to the baseline. Furthermore, for few-shot linear probing, we propose a novel approach that bridges the gap between zero-shot and few-shot performance, substantially improving over prior work.
Models are available at \href{https://github.com/facebookresearch/diht}{github.com/facebookresearch/diht}.
\end{abstract}
\vspace{-2.5em}
\section{Introduction}
\label{sec:intro}
\vspace{-0.5em}

An increasingly popular paradigm in multimodal learning is contrastive pre-training~\cite{cly+20,lyl+20,rkh+21,jyx+21,lsg+21,wyy+21,yhh+21,lcc+21}, which involves training multimodal models on very large-scale noisy datasets of image-text pairs sourced from the web. It has been shown to be incredibly effective for a variety of vision-language tasks without any task-specific fine-tuning (\ie, zero-shot), such as image classification~\cite{rds+15}, text and image retrieval~\cite{lmb+14,pwl+15}, visual question answering~\cite{gks+17}, among several others. In this paper, we study the problem of contrastive pre-training for dual-encoder architectures~\cite{rkh+21} with the objective of improving image-text alignment for {\em zero-shot} tasks. We revisit three important aspects of the contrastive pre-training pipeline -- noise in datasets, model initialization, and contrastive training, and present strategies that significantly improve model performance on a variety of zero-shot benchmarks, see Figure~\ref{fig:teaser}.

\begin{figure}[t]
\centering
\pgfplotstableread{
    samples     l_in    l_coco_t2i  l_coco_i2t  l_flckr_t2i     l_flckr_i2t     l_filt_in   l_filt_coco_t2i     l_filt_coco_i2t     l_filt_flckr_t2i    l_filt_flckr_i2t
    4           0.608   0.337       0.521       0.593           0.777           0.615       0.376               0.559               0.665               0.832
    8           0.634   0.357       0.534       0.608           0.794           0.633       0.388               0.564               0.680               0.832
    16          0.645   0.368       0.551       0.628           0.812           0.644       0.393               0.560               0.681               0.854
    32          0.651   0.373       0.553       0.632           0.804           0.648       0.400               0.574               0.684               0.843
}{\filtering}

\pgfplotstableread{
    id  clip_in     clip_coco_t2i   clip_flickr_t2i     ours_in     ours_coco_t2i   ours_flickr_t2i
    1   0.634       0.314           0.595               0.680       0.406           0.686
    2   0.684       0.337           0.633               0.722       0.433           0.729
    3   0.766       0.377           0.686               0.779       0.493           0.782
}{\clipvsours}

\pgfplotstableread{
    k       ours_b16_pgd    ours_b16_sgd    ours_l14_336_pgd    ours_l14_336_sgd    clip_l14_336_sgd    swag_h14_sgd
    0       0.7220          0.7220          0.7790              0.7790              0.7660              nan
    1       0.7235          0.4463          0.7812              0.5001              0.3896              0.5701
    5       0.7367          0.6678          0.8001              0.7393              0.6642              0.7521
    10      0.7432          0.7079          0.8046              0.7753              0.7280              0.7820
    25      0.7578          0.7444          0.8125              0.8076              0.7768              0.8078
    50      0.7777          0.7658          nan                 nan                 nan                 nan
    100     0.7897          0.7821          nan                 nan                 nan                 nan
}{\fewshot}
\begin{tabular}{ccc}
\hspace{-1.5em}
\begin{tikzpicture}
    \begin{axis}[%
        ylabel near ticks, ylabel shift = -5pt, yticklabel pos=left,
        xlabel near ticks, xlabel shift = -5pt,
        font=\scriptsize,
        width=0.425\linewidth,
        height=0.5\linewidth,
        title={ImageNet1K},
        title style = {yshift = -5pt},
        xlabel={Model Complexity},
        ylabel={Accuracy@1},
        legend pos=south east,
        legend style={cells={anchor=west}, font =\tiny, fill opacity=0.8, row sep=-2.5pt, xshift=0.15em, yshift=-0.2em},
        ymin = 60,
        ymax = 80,
        xmin = 0.8,
        xmax = 3.2,
        grid=both,
        xtick={1,2,3},
        xticklabels={B/32, B/16, L/14\\@336},
        ytick={5,10,...,100},
        tick label style ={font=\scriptsize},
        xticklabel style={font=\scriptsize,align=left},
    ]  
        \addplot[color=darkred, solid, mark=star, mark size=2, line width=1] table[x=id, y expr={100*\thisrow{ours_in}}] \clipvsours; \leg{\ours};
        \addplot[color=darkblue, solid, mark=x, mark size=2, line width=1] table[x=id, y expr={100*\thisrow{clip_in}}] \clipvsours; \leg{CLIP};
    \end{axis}
\end{tikzpicture}
\hspace{-2em}
&
\begin{tikzpicture}
    \begin{axis}[%
        ylabel near ticks, ylabel shift = -5pt, yticklabel pos=left,
        xlabel near ticks, xlabel shift = -5pt,
        font=\scriptsize,
        width=0.425\linewidth,
        height=0.5\linewidth,
        title={COCO (T2I)},
        title style = {yshift = -5pt},
        xlabel={Model Complexity},
        ylabel={Recall@1},
        legend pos=north west,
        legend style={cells={anchor=west}, font =\tiny, fill opacity=0.8, row sep=-2.5pt, xshift=-0.15em, yshift=0.15em},
        ymin = 30,
        ymax = 50,
        xmin = 0.8,
        xmax = 3.2,
        grid=both,
        xtick={1,2,3},
        xticklabels={B/32, B/16, L/14\\@336},
        ytick={5, 10, ..., 100},
        tick label style ={font=\scriptsize},
        tick label style ={font=\scriptsize},
        xticklabel style={font=\scriptsize,align=left},
    ]  
        \addplot[color=darkblue, solid, mark=x, mark size=2, line width=1] table[x=id, y expr={100*\thisrow{clip_coco_t2i}}] \clipvsours; 
        \addplot[color=darkred, solid, mark=star, mark size=2, line width=1] table[x=id, y expr={100*\thisrow{ours_coco_t2i}}] \clipvsours; 
    \end{axis}
\end{tikzpicture}
\hspace{-2em}
&
\begin{tikzpicture}
    \begin{axis}[%
        ylabel near ticks, ylabel shift = -5pt, yticklabel pos=left,
        xlabel near ticks, xlabel shift = -5pt,
        font=\scriptsize,
        width=0.425\linewidth,
        height=0.5\linewidth,
        title={Flickr (T2I)},
        title style = {yshift = -5pt},
        xlabel={Model Complexity},
        ylabel={Recall@1},
        legend pos=north west,
        legend style={cells={anchor=west}, font =\tiny, fill opacity=0.8, row sep=-2.5pt, xshift=-0.15em, yshift=0.15em},
        ymin = 59,
        ymax = 79,
        xmin = 0.8,
        xmax = 3.2,
        grid=both,
        xtick={1,2,3},
        xticklabels={B/32, B/16, L/14\\@336},
        ytick={59, 64, ..., 100},
        tick label style ={font=\scriptsize},
        tick label style ={font=\scriptsize},
        xticklabel style={font=\scriptsize,align=left},
    ]  
        \addplot[color=darkblue, solid, mark=x, mark size=2, line width=1] table[x=id, y expr={100*\thisrow{clip_flickr_t2i}}] \clipvsours; 
        \addplot[color=darkred, solid, mark=star, mark size=2, line width=1] table[x=id, y expr={100*\thisrow{ours_flickr_t2i}}] \clipvsours; 
    \end{axis}
\end{tikzpicture}
\end{tabular}
\vspace{-1.5em}
\caption{
    \ours trained on 438M LAION-CAT samples \vs CLIP~\cite{rkh+21} trained on 400M OpenAI samples.
    \label{fig:teaser}
}
\vspace{-1.5em}
\end{figure}
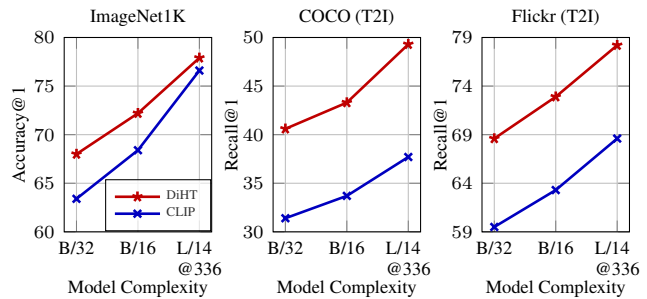

Most image-text datasets are noisy and poorly-aligned. Few recent efforts~\cite{iww+21} have tried to clean the noise by filtering samples based on alignment scores from an existing model like CLIP~\cite{rkh+21}. However, this approach is limited by the biases and flaws of the model itself. On the other hand, momentum-based approaches~\cite{lsg+21} to reduce noise are infeasible for large-scale training due to their increased compute and memory requirements. To this end, we provide a scalable and effective approach titled \textbf{C}omplexity, \textbf{A}ction and \textbf{T}ext-spotting ({\bf CAT}) filtering. CAT is a filtering strategy to select only informative text-image pairs from noisy web-scale datasets. We show that training on a CAT-filtered version of large-scale noisy datasets such as LAION~\cite{sbv+22} can provide up to {\bf 12\%} relative improvements across vision-language tasks despite removing almost {\bf 80\%} of the training data, see Section~\ref{sec:exp_ablation} and Table~\ref{tab:fitering} for more details.

A common strategy~\cite{pdgk21,zwm+22} to further improve multimodal training is to warm-start it with image and text models pre-trained at large scale on their respective modalities. However, due to the increased noise in image-text data, fine-tuning the entire model undermines the benefits of the warm-start. One can alternatively use model freezing strategies like locked-image tuning~\cite{zwm+22}, but they are unable to adapt to the complex queries present in multimodal problems (\eg, cross-modal retrieval) and the models perform poorly on retrieval benchmarks (see Section~\ref{sec:exp_ablation}). We propose an entirely different approach, {\em concept distillation} (CD), to leverage strong pre-trained vision models. The key idea behind {\em concept distillation} is to train a linear classifier on the image encoder to predict the distilled concepts from a pre-trained teacher model, inspired by results in weakly-supervised large-scale classification~\cite{mgr+18,sga+22}.

Finally, we revisit the training objective: almost all prior work has utilized {\em noise-contrastive estimation} via the InfoNCE loss~\cite{olv18}, shortcomings have been identified
in the standard InfoNCE formulation~\cite{crl+20,ksp+20}. We demonstrate that by using a {\em model-based} importance sampling technique to emphasize harder negatives, one can obtain substantial improvements in performance. 

A summary of our pipeline is available in Figure~\ref{fig:pipeline}.
Our combined approach obtains significant improvements over the baseline for dual-encoder architectures on an elaborate benchmark of 29 tasks. 
Specifically, with the ViT-B/16~\cite{dbk+20} architecture, we improve zero-shot performance on {\bf 20 out of 29 tasks}, over CLIP training on the LAION-2B dataset~\cite{iww+21,sbv+22}, despite training on a subset that is {\bf 80\%} smaller, see Figure~\ref{fig:barlaion}.
Furthermore, we demonstrate that even when trained with the smaller (but relatively less noisy) pretraining dataset {\bf PMD}, our performance is better on {\bf 28 out of 29 tasks} than CLIP trained on the same data, often with a large margin, see Figure~\ref{fig:barpmd}.

Additionally, we present a simple yet effective approach to maintain the performance continuum as one moves from zero-shot to few-shot learning in the low data regime. Prior work~\cite{rkh+21} has shown a substantial drop in performance as one moves from zero-shot to $k$-shot learning, which is undesirable for practical scenarios. We propose an alternate linear probing approach that initializes the linear classifier with zero-shot text prompts and ensures that final weights do not drift away too much via projected gradient descent~\cite{b04}. On ImageNet1K, we show huge improvements over prior work for small $k$ values. For example, our approach improves 5-shot top-1 accuracy by an absolute margin of {\bf 7\%} (see Figure~\ref{fig:few}) compared to the baseline strategy of linear probing with a random initialization.

\begin{figure}[t]
    \centering
    \includegraphics[width=1.0\linewidth]{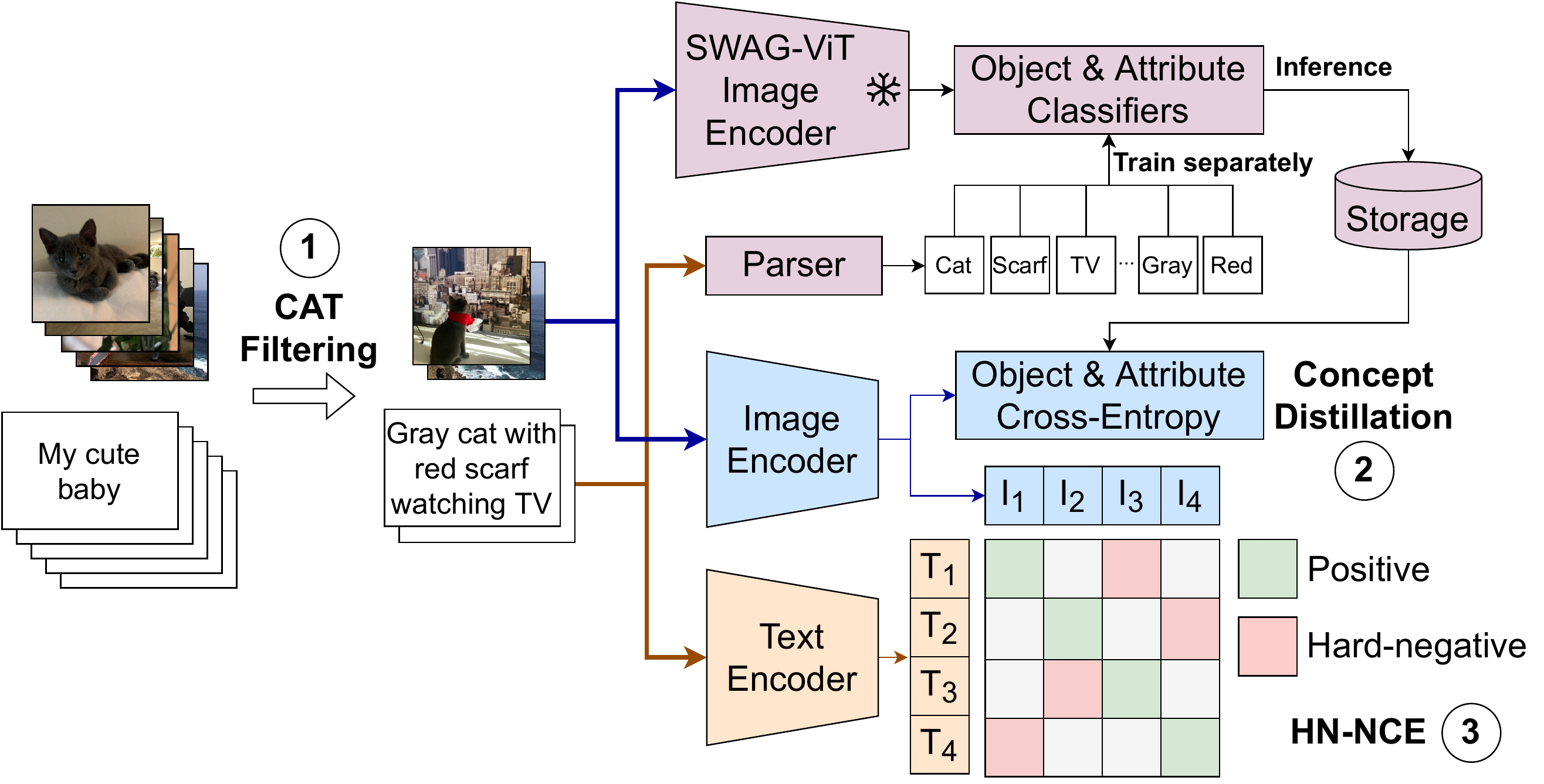}
    \vspace{-2em}
    \caption{Summary of our pipeline. We propose improvements to the standard vision-language pre-training: (1) Complexity, Action and Text-spotting (CAT) filtering that removes non-informative text-image pairs; (2) Concept distillation from a frozen ({\tiny \SnowflakeChevron}) pre-trained image encoder; (3) Hard-negative contrastive loss.}
    \label{fig:pipeline}
\vspace{-1em}
\end{figure}

\section{Related work}
\label{sec:related}

\paragraph{Dataset curation for contrastive pretraining.}
Large-scale contrastive pretraining~\cite{cly+20,lyl+20,rkh+21,jyx+21,lsg+21,wyy+21,yhh+21,lcc+21} typically requires dataset sizes of the order of hundreds of millions to billions. 
Seminal works in this area, \eg, CLIP~\cite{rkh+21} and ALIGN~\cite{jyx+21}, have largely relied on image-text pairs crawled from the web. Subsequently, versions of large-scale image-text datasets have been created but not released publicly, including WIT-400M~\cite{rkh+21}, ALIGN-1.8B~\cite{jyx+21}, FILIP-340M~\cite{yhh+21}, FLD-900M~\cite{lcc+21}, BASIC-6.6B~\cite{pdgk21}, PaLI-10B~\cite{cwc+22}. These  datasets often use unclear or primitive cleaning strategies, \eg, removing samples with short or non-English captions. 
Recently, LAION-400M \cite{laion400m-sch+21} used CLIP-based scores to filter down a large dataset. The authors later released an English-only LAION-2B and a LAION-5B unfiltered dataset sourced from Common Crawl\footnote{\href{https://commoncrawl.org}{commoncrawl.org}}.
Apart from LAION-400M and BLIP~\cite{llxh22} which uses the bootstrapped image-grounded text encoder to filter out noisy captions, there has not been a significant investment in systematic curation strategies to improve zero-shot alignment performance on large-scale pretraining. 
In contrast to the previous work, we use quality-motivated filters that retain images whose captions are sufficiently complex, contain semantic concepts (actions), and do not contain text that can be spotted in the image \cite{ksl+21}.

\vspace{-1em}
\paragraph{Distillation from pre-trained visual models.}
Knowledge distillation \cite{hvd15} aims to transfer knowledge from one model to another and has been used in many contexts ranging from improving performance and efficiency \cite{bcn+06,ccy+17,lll+20,rbk+14,xlh+20,tcd+21,sx22} to improving generalization capabilities \cite{dws+19,lys+17,lyl+20}.
Several approaches use self-distillation to improve performance with lower computational overhead \cite{hc+19,xl+19,ypl+20}. For vision and language pre-training, \cite{ach+22, lsg+21, kvy+22} use soft-labels computed using embeddings from a moving average momentum model with the goal to reduce the adverse effects of noisy image-text pairs in the training data. Our concept distillation approach is a cheaper and more effective alternative, since it does not require us to run the expensive teacher model throughout the training\footnote{Distillation targets can be pre-computed and stored.} while retaining the most useful information from the visual concepts.

Another approach to take advantage of pre-trained visual models is to use them to initialize the image encoder, and continue pre-training either by locking the image encoder~\cite{pdgk21,zwm+22} or fine-tuning~\cite{pdgk21}.
However, these approaches lack the ability to align complex text to a fully-trained image encoder, and thus perform poorly on multimodal tasks, \eg cross-modal retrieval (see Section~\ref{sec:exp_sota}).\vspace*{0.05in}

\paragraph{Contrastive training with hard negatives.}
{\em Noise-contrastive estimation} (NCE)~\cite{gh10} is the typical objective for vision-text learning, with applications across large-scale multimodal alignment~\cite{rkh+21, jyx+21, cly+20,lyl+20} and unsupervised visual representation learning~\cite{mm20, hfw+20}. Several lines of work have studied the shortcomings of the original InfoNCE objective~\cite{olv18}, specifically, the selection and importance of negative samples. Chuang~\etal\cite{crl+20} present a debiasing approach to account for false negatives at very large batch sizes, typical in large-scale pretraining. Kalantidis~\etal\cite{ksp+20} present a MixUp approach to improve the quality of hard negative samples for unsupervised alignment. Using model-specific hard negatives in the training objective is proven to reduce the estimation bias of the model as well~\cite{zs21}. Contrary to prior semi-supervised work, we extend the model-based hard negative objective, first proposed in Robinson~\etal\cite{rcsj21} to multimodal alignment.

\section{Method}
\label{sec:method}

\paragraph{Background.}
We consider the task of contrastive image-text pretraining. Given a dataset $\cD = \{(I_i, T_i)\}_{i=1}^N$ of image-text pairs, we want to learn a dual encoder model $\bphi = \{\bphi_{\text{image}}, \bphi_{\text{text}}\}$, where $\bphi_{\text{image}}$ represents the image encoder, and $\bphi_{\text{text}}$ denotes the text encoder.
We use the shorthand $\bx = \bphi_{\text{image}}(I)$ and $\bt = \bphi_{\text{text}}(T)$ to denote the encoded images and texts, respectively, for an image-text pair $(I, T)$. We will now describe the three crucial components of our approach followed by the final training objective. 

\subsection{Complexity, Action, and Text (CAT) filtering}
\label{sec:method_filtering}
Our complexity, action, and text spotting (CAT) filtering is a combination of two filters: a caption complexity filter that removes image-caption pairs if a caption is not sufficiently complex, and an image filter that removes pairs if the image contains text matching the caption to prevent polysemy during alignment. We use the LAION-2B \textit{pre-cleaned} obtained after using filters\footnote{Not-suitable-for-view images and toxic captions.} in \cite{sph+22} as the base dataset.\vspace*{0.05in}

\paragraph{Filtering captions via complexity \& actions.}
Noisy web-scale datasets do not have any semantic-based curation, and hence captions can be irrelevant, ungrammatical and unaligned. Our motivation is to decrease such noise by simply selecting captions that possess sufficient complexity, so that the training distribution matches the target tasks.
To this end, we build a fast rule-based parser that extracts objects, attributes and action relations (see Figure~\ref{fig:object-data-parser} for an example) from text and we use the resulting semantic graph to estimate the complexity of the image captions. Specifically, we define the complexity of a caption as the {\em maximum number of relations to any object} present in the parse graph. For example, in the caption ``A black cat is chasing a small brown bird,'' the object ``bird'' has the attributes ``small'', ``brown'' and ``A black cat is chasing'', and hence, the complexity of the caption is C$3$. We only retain samples that at least have a C$1$ caption complexity. To further remove pairs likely containing products, we filter out captions if they do not contain at least one action (as obtained from the parser).

\begin{figure}
    \centering
    \includegraphics[width=0.8\linewidth]{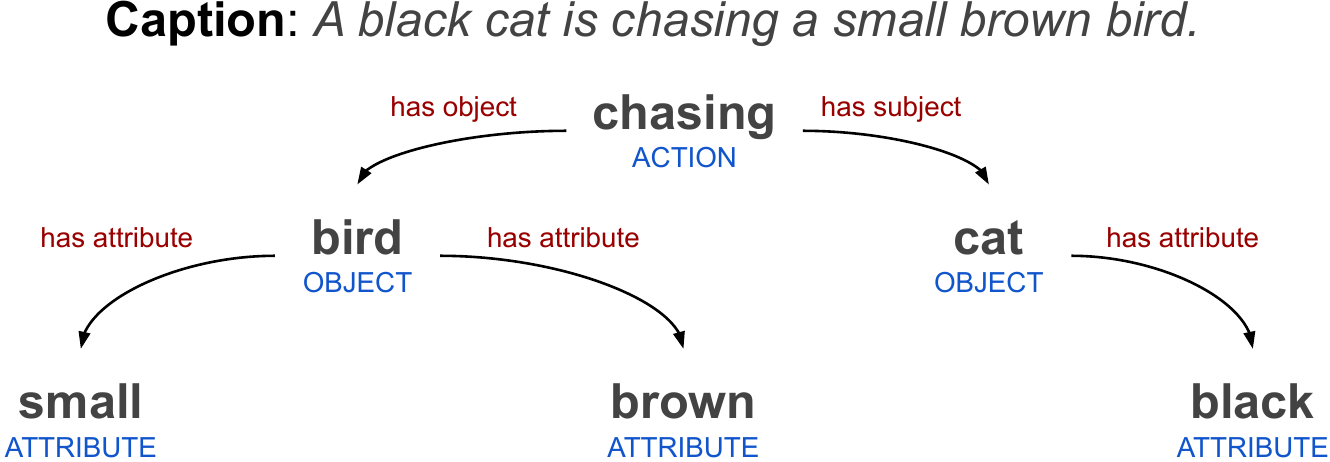}
    \vspace{-0.5em}
    \caption{An example caption and its parse. The caption has C$3$ complexity (due to {\em bird}) and has 1 action ({\em chasing}).}
    \label{fig:object-data-parser}
\end{figure}

\paragraph{Filtering images via text-spotting.}
Image-caption pairs in web-scale datasets often display part of the caption as text in the image (on visual inspection, we found up to $\sim$30\% such examples for LAION~\cite{sbv+22}). Minimizing the objective, in these cases, can correspond to spotting text (\eg, optical character recognition) rather than the high-level visual semantics (\eg, objects, attributes) we would like the model to align to. This will reduce performance on object-centric and scene-centric downstream zero-shot tasks, and hence we remove such images from the training set using an off-the-shelf text spotter~\cite{ksl+21}.
We remove image-text pairs with a text spotting confidence of at least $0.8$ and at least $5$ predicted characters matching the caption in a sliding window. We observe (by inspection) that this approach is efficient at identifying images with text, and failure cases are primarily in non-English text. Filtering with multilingual text spotters trained can fix this issue, however, we leave this as future work. Filtering statistics can be found in the supplement. 

\subsection{Concept distillation}
\label{sec:method_na}

Recognizing visual concepts in images that correspond to \objects and attributes in corresponding captions is crucial for alignment. We therefore propose to distill these concepts from a pre-trained teacher model to our image encoder. Specifically, we add two auxiliary linear classifiers on top of the encoded image embeddings $\bx$ to predict {\em(i)} \objects and {\em(ii)} \emph{visual} attributes and use the teacher model to generate the pseudo-labels for training them. These classifiers are trained jointly with the contrastive loss.

We parse image captions using a semantic parser that extracts \objects and attributes from text (Section~\ref{sec:method_filtering}) and use these as pseudo-labels. We then train the linear classifiers on the teacher model embeddings with a soft-target cross-entropy loss~\cite{gbc16}, after square-root upsampling low-frequency concepts~\cite{mgr+18}. It is important to freeze the backbone of the teacher model to make sure we retain the advantages of using a stronger model for distillation. For each image, we then use these trained linear classifiers to generate two softmax probability vectors -- $\bp^{\text{obj}}$ for \objects, and $\bp^{\text{attr}}$ for attributes, respectively. To minimize the storage overhead, we further sparsify them by retaining only the top-$k$ predicted class values and re-normalizing them to generate the final pseudo-labels. During multimodal training, we use the cross-entropy loss with these pseudo-label vectors as targets. Unless specified otherwise, we use the ViT-H/14~\cite{dbk+20} architecture pretrained from SWAG~\cite{sga+22} as the teacher model. See Section~\ref{sec:exp_ablation} and the supplementary material for ablations on the effect of different backbones and retaining top-$k$ predictions, and further details.

There are several advantages of our concept distillation approach. First, the teacher predictions capture correlations from the strong vision encoding, making them more informative as labels compared to the captions themselves. The captions are limited to a few \objects and attributes, while the teacher predictions yield a more exhaustive list. Moreover, our approach reaps the benefits of the recently proposed and publicly-available strong unimodal vision models more effectively than other distillation approaches, as training linear classifiers on a frozen teacher model is inexpensive. After predictions are stored, we discard the teacher model and thus bypass the memory and compute limitations of simultaneously running the student and teacher model in standard distillation approaches~\cite{hvd15,tcd+21}, which is critical for large teacher models. We demonstrate empirically (see Section~\ref{sec:exp_ablation}) that our strategy works better than distilling teacher embeddings directly. Additionally, compared to approaches that warm-start the image encoder with pre-trained models, our method can leverage higher capacity teacher models without difficulty and unlike locked-image tuning~\cite{pdgk21,zwm+22}, our approach gives the flexibility of training the image encoder for better alignment, while retaining the strength of the pre-trained visual features.

\subsection{Multimodal alignment with hard negatives}
\label{sec:method_hn}
Contrastive learning~\cite{olv18} has quickly become the de-facto approach for multimodal alignment, where most prior work focuses on the multimodal InfoNCE~\cite{olv18} objective, given for any batch $\bX = \{(\bx_i,\bt_i)\}_{i=1}^n$ of featurized image-text pairs as (for some learnable temperature $\tau > 0$),
{\small
\begin{align*}
\cL_{\text{NCE}}(\bX) = -\sum_{i=1}^n\left[\log\frac{e^{\bx_i^\top\bt_i/\tau}}{\sum_{j}e^{\bx_i^\top\bt_j/\tau}}+\log\frac{e^{\bx_i^\top\bt_i/\tau}}{\sum_{j}e^{\bx_j^\top\bt_i/\tau}}\right].
\end{align*}
}
While this approach has enjoyed immense success in multimodal alignment~\cite{rkh+21,jyx+21}, when learning from large-scale noisy datasets, uniform sampling as applied in {\em noise-contrastive estimation} can often provide negative samples that are not necessarily discriminative, necessitating very large batch sizes. For the problem of contrastive self-supervised learning, Robinson~\etal~\cite{rcsj21} propose an importance-sampling approach to reweight negative samples within a batch so that ``harder'' negatives are up-sampled in proportion to their difficulty. We present a similar strategy for multimodal alignment. Specifically, for some $\alpha \in (0, 1], \beta \geq 0$, we propose the following {\em hard-negative} noise contrastive multimodal alignment objective:
{\small
\begin{align*}
    \cL_{\text{HN-NCE}}(\bX) = - &\sum_{i=1}^n \log\left[\frac{e^{\bx_i^\top\bt_i/\tau}}{\alpha\cdot e^{\bx_i^\top\bt_i/\tau}+\underset{j\neq i}{\sum} e^{\bx_i^\top\bt_j/\tau} w^{i\rightarrow t}_{\bx_i, \bt_j}}\right] \\
    -&\sum_{i=1}^n \log\left[\frac{e^{\bx_i^\top\bt_i/\tau}}{\alpha\cdot e^{\bx_i^\top\bt_i/\tau}+\underset{j\neq i}{\sum}  e^{\bx_j^\top\bt_i/\tau}w^{t\rightarrow i}_{\bx_j, \bt_i}}\right].
\end{align*}
    }
Where the weighing functions are given as\footnote{We normalize by $n-1$ as this is the number of negatives.}: 
{\small
\begin{align*}
    w^{i\rightarrow t}_{\bx_i, \bt_j} = \frac{(n-1)\cdot e^{\beta\bx_i^\top\bt_j/\tau}}{\sum_{k\neq i}e^{\beta\bx_i^\top\bt_k/\tau}}, 
    w^{t\rightarrow i}_{\bx_j, \bt_i} = \frac{(n-1)\cdot e^{\beta\bx_j^\top\bt_i/\tau}}{\sum_{k\neq i} e^{\beta\bx_k^\top\bt_i/\tau}}.
\end{align*}
}
The weights $w_\beta$ are designed such that difficult negative pairs (with higher similarity) are emphasized, and easier pairs are ignored. Furthermore, $\alpha$ rescales the normalization with the positive terms to account for the case when false negatives are present within the data. The form of weights $w_\beta$ is an unnormalized von Mises-Fisher distribution~\cite{mj00} with concentration parameter $\beta$. Observe that we obtain the original objective when setting $\alpha=1$ and $\beta=0$. There are several key differences with the original formulation of~\cite{rcsj21} and the {\sc HN-NCE} objective presented above. First, we utilize only cross-modal alignment terms, instead of the unimodal objective presented in~\cite{rcsj21}. Next, we employ separate penalties for text-to-image and image-to-text alignment. Finally, we incorporate a learnable temperature parameter $\tau$ to assist in the learning process. We discuss our design choices in more detail with additional theoretical and experimental justifications in the supplementary material.

\subsection{Training objective} 
\label{sec:objective}

For any batch $\bX = \{(\bx_i, \bt_i)_{i=1}^n\}$ of $n$ image-text pairs, we minimize the following objective:
\begin{align*}
    &\cL_{\text{HN-NCE}}(\bX) + \cL_{\text{CE-O}}(\bX) + \cL_{\text{CE-A}}(\bX), \text{where,} &\\
    &\cL_{\text{CE-O}}(\bX) = \sum_{i=1}^n \textsc{Cross-Entropy}(\bp^{\text{\obj}}_{i}, f_{\text{\obj}}(\bx_i)), \text{and,}&\\  
    &\cL_{\text{CE-A}}(\bX) = \sum_{i=1}^n \textsc{Cross-Entropy}(\bp^{\text{attr}}_{i}, f_{\text{attr}}(\bx_i)).&
\end{align*}
Here, both $f_{\text{\obj}}$ and $f_{\text{attr}}$ are linear classifiers, the vectors $\bp^{\text{\obj}}, \bp^{\text{attr}}$ are the top-k predicted \objects and attributes from the teacher model (Section~\ref{sec:method_na}), and $\cL_{\textsc{hn-nce}}$ is the hard-negative contrastive alignment loss (Section~\ref{sec:method_hn}). 

\section{Experiments}
\label{sec:experiments}
Here we evaluate our approach across a broad range of vision and vision-language tasks.
We provide extensive ablations on 29 tasks over the design choices in Section~\ref{sec:exp_ablation}, and compare with state-of-the-art approaches on popular zero-shot benchmarks in Section~\ref{sec:exp_sota}.
Finally, we present an alternate approach to do few-shot classification with prompt-based initialization in Section~\ref{sec:exp_few}.

\subsection{Experimental setup}
\label{sec:exp_setup}

\paragraph{Training datasets.}
We use a 2.1B English caption subset of the LAION-5B dataset~\cite{sbv+22}.
Prior to training, we filter out sample pairs with NSFW images, toxic words in the text, or images with a watermark probability larger than $0.5$, following~\cite{sph+22}.
This leaves us with 1.98B images, which we refer to throughout the paper as the LAION-2B dataset. Additionally, we explore training our models on a collection of Public Multimodal Datasets (PMD) from~\cite{shg+22}. PMD contains training splits of various public datasets. After downloading\footnote{Downloaded following \href{https://huggingface.co/datasets/facebook/pmd}{huggingface.co/datasets/facebook/pmd}.} the data we are left with 63M (\vs 70M reported in~\cite{shg+22}) image-text pairs due to missing samples and SBU Captions~\cite{okb11} (originally in PMD) going offline.

\paragraph{Training details.}
For our model architecture, we closely follow CLIP by Radford~\etal~\cite{rkh+21}.
We utilize Vision Transformers (ViT)~\cite{dbk+20} for images and Text Transformers~\cite{vsp+17} for captions.
We experiment with 3 different architectures, denoted as B/32, B/16, and L/14, where 32, 16, and 14 denote the input image patch size. See the supplementary for architecture details.
For distillation and fine-tuning experiments, we utilize the public SWAG-ViT models~\cite{sga+22}, pre-trained with weak supervision from hashtags.

We use the Adam~\cite{kb15} optimizer with a decoupled weight decay~\cite{lh19} and a cosine learning rate schedule~\cite{lh17}. The input image size is 224$\times$224 pixels. To accelerate training and save memory, we use mixed-precision training~\cite{mna+18}.
All hyperparameters are presented in the supplementary.
They are selected by training B/32 on a small scale setup, and reused for all architectures.
For \objects and attributes classifiers, we found that scaling the learning rate by 10.0 and weight decay by 0.01 gave better results.
We train our models on 4B, 8B, 16B, and 32B total samples. 
For ViT-L/14, we further train the model at a higher 336px resolution for 400M samples, denoting this model as L/14@336.
We trained L/14 for 6 days on 512 A100 GPUs with 16B processed samples for a total of $7.4 \times 10^4$ GPU hours.

\paragraph{Evaluation benchmarks.}
We evaluate our models on a zero-shot benchmark of 29 tasks: (i) 17 image classification, (ii) 10 cross-modal retrieval, (iii) 2 visual question answering. Dataset details are presented in the supplement.

\subsection{Ablations on zero-shot benchmarks}
\label{sec:exp_ablation}

In this section, we ablate our three pretraining contributions: dataset filtering, distillation from \objects and attributes predictions, and, hard negative contrastive objective.
Ablations are performed over zero-shot Accuracy@1 on the ImageNet1K~\cite{rds+15} (IN) validation set, text-to-image (T2I) and image-to-text (I2T) zero-shot Recall@1 on the COCO~\cite{puc+20} and Flickr~\cite{pwl+15} test sets.
We also report the change in accuracy (\%) over 29 zero-shot tasks between our model and baselines.
For a fair comparison, we train all approaches presented in this section (including baselines).

\vspace{-1em}
\paragraph{Effect of dataset filtering.}
We apply our filters, as well as filtering based on CLIP~\cite{rkh+21} alignment score ($<$0.35), and ablate the baseline performance, without distillation or hard negative contrastive training, in Table~\ref{tab:fitering} for ViT-B/32 model architecture.
All models see 4B total samples during training, while the number of unique samples drops after each filtering step.
Complexity filter (C) in row (3) reduces the dataset size by around 270M, while slightly increasing image-text alignment as observed on T2I task.
Next, action filter (A) in row (4) reduces the size by more than 1B, while it has a large benefit in aligning complex text.
However, as expected, it hurts performance on object-centric ImageNet.
Finally, text-spotting (T) filter in row (5) boosts alignment across the board, due to the fact that it removes the need to learn a bimodal visual representation of the text.
We also compare with filtering based on CLIP score in row (2), which was selected such that the dataset size is comparable to ours, and show that it is too strict and removes plenty of useful training pairs, thus hurting the performance.
Finally, LAION-CAT, with only {\bf 22}\% of the original dataset size, significantly boosts image-text zero-shot performance. We also observed that gains hold as we train for longer training schedules. See the supplementary for details.

\begin{table}[t]
\centering
\caption{
    Evaluating effect of using LAION-2B subset filtered on complexity (C), actions (A), and text-spotting (T). 
    CLIP denotes filtering pairs with CLIP score bellow 0.35.
    Evaluation performed on ViT-B/32 model architecture trained for 4B processed samples.
}
\label{tab:fitering}
\vspace{-1em}
\def\arraystretch{0.85}  
\setlength{\tabcolsep}{1mm}  
\def\cw{0.3cm}
\def\cww{0.7cm}
\def\hs{1mm}
\footnotesize{
    \begin{tabular}{C{0.2cm}C{0.6cm}C{\cw}C{\cw}C{\cw}C{0.8cm}C{\cww}C{\cww}C{\cww}C{\cww}C{\cww}}
        \toprule
        \multirow{2}{*}{\#} & \multicolumn{4}{c}{Filter} & \multirow{2}{*}{Size} & \multirow{2}{*}{IN} & \multicolumn{2}{c}{COCO} & \multicolumn{2}{c}{Flickr} \\
        \cmidrule(l{\hs}r{\hs}){2-5}\cmidrule(l{\hs}r{\hs}){8-9}\cmidrule(l{\hs}r{\hs}){10-11}
         & \scriptsize{CLIP} & C & A & T & & & T2I & I2T & T2I & I2T \\
        \midrule
        1 & & & & & 1.98B & 60.8 & 33.7 & 52.1 & 59.3 & 77.7 \\
        2 & \checkmark & & & & 440M & 52.5 & 29.8 & 46.1 & 54.8 & 72.0 \\
        3 & & \checkmark & & & 1.71B & 60.8 & 33.9 & 52.5 & 60.8 & 77.8 \\
        4 & & \checkmark & \checkmark & & 642M & 58.7 & 35.9 & 53.8 & 64.3 & 82.0 \\
        5 & & \checkmark & \checkmark & \checkmark & 438M & {\bf61.5} & {\bf37.6} & {\bf55.9} & {\bf66.5} & {\bf83.2} \\
        \bottomrule
    \end{tabular}
}
\vspace{-1em}
\end{table}

\vspace{-1em}
\paragraph{Effect of distillation approach.} 
To understand the effect of direct distillation from a pre-trained SWAG-ViT visual encoder~\cite{sga+22}, we investigate two baseline approaches:

\noindent{(1) \it{Embedding distillation (ED)}} borrows from SimSiam~\cite{xk21} and uses an auxiliary negative cosine similarity loss between the image representation from the student visual encoder and the pre-trained SWAG model.\\
\noindent{(2) \it{Distribution distillation (DD)}} borrows ideas from momentum distillation in ALBEF~\cite{lsg+21} and computes the cross-modal similarities between the SWAG image representation and the student text representation and uses them as soft-labels for student image representation and text alignment. The soft-labels are linearly combined with the hard $0-1$ labels before applying the InfoNCE~\cite{olv18} loss. 

A comparison of our distillation from predicted concepts (CD) with the aforementioned distillation approaches is presented in Table~\ref{tab:distillation} (upper section). Note that for a fair comparison, we do not use our hard-negative contrastive loss for these experiments.
Our distillation approach performs the best, even though it has virtually no training overhead as the predicted concepts are pre-computed, while, \eg, ED is 60\% slower with an 8\% increase in GPU memory due to the need of running an additional copy of the vision tower.
One could pre-compute embeddings for ED and DD as well, but that increases dataset size by 1.2TB and creates a data loading bottleneck, while our pre-computed predictions take only 32.6GB additional storage space when saving the top-10 predictions (see supplementary).
We additionally show that our approach is robust to the number of top-$k$ predictions used, details in the supplementary.

One could also use an external unimodal image model and fine-tune it on the image-text alignment task instead of using distillation.
We follow \cite{zwm+22} and explore three fine-tuning options as baselines: (i) locked-image tuning (LiT) where the image encoder is locked, and only the text encoder is trained, (ii) fine-tuning (FT) where the image encoder is trained with a learning rate scaled by 0.01 compared to the text encoder, (iii) fine-tuning with delay (FT-delay) where the image encoder is locked for half of the pre-training epochs following (i), and then fine-tuned for the rest following (ii).
Results of these setups are ablated in Table~\ref{tab:distillation} (lower section).
LiT \vs FT is a trade-off between strong performance on image recognition tasks (as measured with ImageNet1K) and better image-text alignment (as measured by COCO and Flickr).
Locking the image encoder makes the alignment very hard to achieve, but fine-tuning it hurts its original image recognition power. 
On the other hand, we show that our concept distillation is the best of both worlds, it surpasses LiT or FT in 4 out of 5 metrics.
Another drawback of FT is that it requires the same architecture in the final setup, while CD can be effortlessly combined with any architecture or training setup, by using stored predictions as metadata. To conclude, unlike related approaches, our proposed distillation: (i) has almost no cost at training, (ii) is architecture agnostic, (iii) improves both image recognition and complex image-text alignment.

\begin{table}[t]
\centering
\caption{
    Evaluating effect of using different initialization or distillation approaches.
    Evaluation performed on ViT-B/16 model architecture trained for 16B processed samples on LAION-CAT.
    Init: Initialization with random or SWAG-B/16 weights.
    ED: Embedding distillation.
    DD: Distribution distillation.
    LiT: Locked image tuning.
    FT: Fine-tuning.
    FT-delay: Locked image tuning for 50\% followed by fine-tuning for the rest.
    CD: Our concept distillation using teacher-predicted \objects and attributes.
}
\vspace{-1em}
\label{tab:distillation}
\def\arraystretch{0.85}  
\setlength{\tabcolsep}{1mm}  
\def\cww{0.7cm}
\def\hs{1mm}
\centering
\footnotesize{
    \begin{tabular}{C{0.3cm}L{1.5cm}C{0.9cm}C{\cww}C{\cww}C{\cww}C{\cww}C{\cww}}
        \toprule
        \multirow{2}{*}{\rotatebox[origin=c]{90}{Init}} & \multirow{2}{*}{Method} & SWAG & \multirow{2}{*}{IN} & \multicolumn{2}{c}{COCO} & \multicolumn{2}{c}{Flickr} \\
        \cmidrule(l{\hs}r{\hs}){5-6}\cmidrule(l{\hs}r{\hs}){7-8}
        & & \scriptsize{(teacher)} & & T2I & I2T & T2I & I2T \\
        \midrule
        \multirow{5}{*}{\rotatebox[origin=c]{90}{Random}}
        & Baseline & --- & 68.7 & 42.8 & 60.5 & 72.8 & {\bf89.7} \\
        & ED & B/16 & 69.2 & 42.6 & 59.4 & 72.8 & 86.8 \\
        & DD & B/16 & 68.6 & 41.8 & 57.4 & 71.7 & 87.0 \\
        & CD (ours) & B/16 & 71.0 & 42.8 & 59.5 & 72.3 & 86.5 \\
        & CD (ours) & H/14 & 72.3 & {\bf43.4} & 60.4 & {\bf73.8} & 87.6 \\
        \midrule
        \multirow{3}{*}{\rotatebox[origin=c]{90}{SWAG}}
        & LiT & --- & {\bf73.0} & 32.5 & 50.6 & 60.8 & 79.6 \\
        & FT & --- & 71.2 & 43.1 & 60.3 & 73.1 & 87.7 \\
        & FT-delay & --- & 72.0 & 42.7 & {\bf60.7} & 72.5 & 86.2 \\
        \bottomrule
    \end{tabular}
}
\end{table}

\vspace{-1em}
\paragraph{Effect of hard negative contrastive training.}
We present the ablation when using hard negative contrastive objective (HN-NCE) in Table~\ref{tab:hn}.
Performance suggests that using the newly proposed loss is beneficial compared to the vanilla InfoNCE, and that its positive effects are complementary to the gains from the proposed distillation from objects and attributes predictions. Please see the supplementary for ablations on the effect of the hyperparameters $\alpha$ and $\beta$.

\begin{table}[t]
\centering
\vspace{-0.5em}
\caption{
    Evaluating effect of using hard negative contrastive loss. 
    Evaluation performed on ViT-B/16 model architecture trained for 16B processed samples on LAION-CAT.
    CD: Our concept distillation using SWAG-H/14 predicted \objects and attributes.
    HN: Our proposed hard negative contrastive loss.
}
\vspace{-1em}
\label{tab:hn}
\def\arraystretch{0.85}  
\setlength{\tabcolsep}{1mm}  
\def\cw{1.0cm}
\def\cww{0.7cm}
\def\hs{1mm}
\centering
\footnotesize{
    \begin{tabular}{C{0.2cm}C{\cw}C{\cw}C{\cww}C{\cww}C{\cww}C{\cww}C{\cww}}
        \toprule
        \multirow{2}{*}{\#} & \multicolumn{2}{c}{Method} & \multirow{2}{*}{IN} & \multicolumn{2}{c}{COCO} & \multicolumn{2}{c}{Flickr} \\
        \cmidrule(l{\hs}r{\hs}){2-3}\cmidrule(l{\hs}r{\hs}){5-6}\cmidrule(l{\hs}r{\hs}){7-8}
         & CD & HN & & T2I & I2T & T2I & I2T \\
        \midrule
        1 & & & 68.7 & 42.8 & 60.5 & 72.8 & {\bf89.7} \\
        2 & \checkmark & & {\bf72.3} & 43.4 & 60.4 & {\bf73.8} & 87.6 \\
        3 & \checkmark & \checkmark & 72.0 & {\bf43.7} &{\bf 62.0} & 73.2 & 89.5 \\
        \bottomrule
    \end{tabular}
}
\vspace{-1em}
\end{table}

\vspace{-1em}
\paragraph{Effect when pre-training on PMD.}
Finally, we analyze our proposed recipes when training visual-language models on a much smaller dataset, \ie PMD with 63M training samples.
Results are shown in Table~\ref{tab:pmd}.
All contributions improve the performance over baseline significantly, hence we conclude that using the proposed pipeline is very beneficial in low-resource training regimes\footnote{PMD is smaller and relatively much cleaner dataset compared to LAION. Hence, we observed that our filtering step is not needed for it.}.
Note that, the PMD dataset contains COCO and Flickr training samples, hence, it is not strictly zero-shot evaluation.
For that reason, we do not compare our models trained on PMD dataset with state-of-the-art models in the following section.
However, we believe these strong findings will motivate usage of our approach on smaller and cleaner datasets, as well.

\begin{table}[t]
\centering
\caption{
    Evaluating effect when pre-training on PMD using our approaches. 
    Evaluation performed on ViT-B/32 and ViT-B/16 models trained for 4B processed samples.
    CD: Our concept distillation using SWAG-H/14 predicted \objects (-O) and attributes (-A).
    HN: Our proposed hard negative contrastive loss.
}
\vspace{-1em}
\label{tab:pmd}
\def\arraystretch{0.85}  
\setlength{\tabcolsep}{1mm}  
\def\cw{0.75cm}
\def\cww{0.7cm}
\def\hs{1mm}
\centering
\footnotesize{
    \begin{tabular}{C{0.3cm}C{0.2cm}C{\cw}C{\cw}C{\cw}C{\cww}C{\cww}C{\cww}C{\cww}C{\cww}}
        \toprule
        \multirow{2}{*}{\rotatebox[origin=c]{90}{Arch.}} & \multirow{2}{*}{\#} & \multicolumn{3}{c}{Method} & \multirow{2}{*}{IN} & \multicolumn{2}{c}{COCO} & \multicolumn{2}{c}{Flickr} \\
        \cmidrule(l{\hs}r{\hs}){3-5}\cmidrule(l{\hs}r{\hs}){7-8}\cmidrule(l{\hs}r{\hs}){9-10}
         & & CD-O & CD-A & HN & & T2I & I2T & T2I & I2T \\
        \midrule
        \multirow{4}{*}{\rotatebox[origin=c]{90}{B/32}}
        & 1 & & & & 49.0 & 28.9 & 50.2 & 62.0 & 80.3  \\
        & 2 & \checkmark & & & 57.8 & 32.2 & 54.0 & 65.6 & 85.7 \\
        & 3 & \checkmark & \checkmark & & 59.7 & 34.4 & 55.7 & 68.3 & 87.8 \\
        & 4 & \checkmark & \checkmark & \checkmark & {\bf62.4} & {\bf37.3} & {\bf60.4} & {\bf71.8} & {\bf89.9} \\
        \midrule
        \multirow{3}{*}{\rotatebox[origin=c]{90}{B/16}}
        & 5 & & & & 54.6 & 33.1 & 55.7 & 67.4 & 85.5 \\
        & 6 & \checkmark & \checkmark & & 65.5 & 37.4 & 59.9 & 72.4 & 88.7 \\
        & 7 & \checkmark & \checkmark & \checkmark & {\bf67.8} & {\bf42.7} & {\bf65.5} & {\bf77.6} & {\bf92.5} \\
        \bottomrule
    \end{tabular}
}
\end{table}

\vspace{-1em}
\paragraph{Zero-shot benchmarks.}
We denote model trained with our proposed concept distillation and hard-negative loss as \ours.
To showcase our model's performance in more detail, we report our \ours-B/16 trained on LAION-CAT with 438M samples \vs CLIP-B/16 baseline trained by us on LAION-2B with 2B samples in Figure~\ref{fig:barlaion}.
Additionally, we report \ours-B/16 \vs CLIP-B/16 baseline, where both models are trained on PMD dataset with 63M samples in Figure~\ref{fig:barpmd}.
When trained on LAION-CAT or LAION-2B, respectively, \ours wins on 20 out of 29 benchmark tasks. 
Impressively, when trained on PMD, \ours wins on 28 out of 29 benchmarks tasks, usually with a very large margin.

\begin{figure}
    \centering
    \vspace{-1em}
    \begin{tikzpicture}
    \begin{axis} [
        width=1.05\linewidth,
        height=4.6cm,
        font=\tiny,
        ybar,
        bar width=6pt,
        axis x line=center,
        axis y line=box,
        x axis line style={-,thick},
        y axis line style={thick},
        ytick align=outside,
        ytick={-45,-40,...,45},
        yticklabel style={rotate=90, yshift=-2pt},
        xtick={1,2,...,9},
        xticklabels={\cite{evw+10}~PascalVOC,\cite{ffp04}~Caltech101,\cite{ksdf13}~StanfordCars,\cite{xeh+16}~SUN397,\cite{blw+14}~Birdsnap,\cite{rkh+21}~Country211,\cite{kh09}~CIFAR10,\cite{kra+20}~OpenImages,\cite{nz08}~Flowers102,},
        xticklabel style={yshift=2pt, anchor=west, rotate=90},
        extra x ticks={10,11,...,29},
        extra x tick style={xticklabel style={yshift=0pt, anchor=east}},
        extra x tick labels={CIFAR100~\cite{kh09},STL10~\cite{cnl11},COCO I2T~\cite{lmb+14},HatefulMemes~\cite{kfm20},OxfordPets~\cite{pvzj12},VQAv2~\cite{gks+17},DTD~\cite{cmk+14},LN-COCO I2T~\cite{puc+20},SNLI-VE~\cite{xldk19},Winoground T2I~\cite{tjb+22},Food101~\cite{bgg14},LN-COCO T2I~\cite{puc+20},COCO T2I~\cite{lmb+14},ImageNet1K~\cite{rds+15},UCF101~\cite{szs12},Flickr T2I~\cite{pwl+15},LN-Flickr I2T~\cite{puc+20},Flickr I2T~\cite{pwl+15},Winoground I2T~\cite{tjb+22},LN-Flickr T2I~\cite{puc+20},},
        xmin = 0.35,
        xmax = 29.65,
        ymin = -20,
        ymax = 20,
        ylabel = {Score Difference (\%)},
        ylabel near ticks,
        ylabel shift = -5pt,
        nodes near coords,
        nodes near coords align={horizontal},
        nodes near coords style={/pgf/number format/fixed,/pgf/number format/fixed zerofill,/pgf/number format/precision=2, rotate=90},
    ]
    \addplot[black,fill=red] coordinates {
        (1.50,-7.62)
        (2.50,-5.63)
        (3.50,-3.07)
        (4.50,-2.81)
        (5.50,-2.22)
        (6.50,-1.94)
        (7.50,-1.86)
        (8.50,-1.38)
        (9.50,-1.04)
    };
    \addplot[black,fill=green] coordinates {
        (9.50,0.25)
        (10.50,0.25)
        (11.50,0.42)
        (12.50,0.49)
        (13.50,0.57)
        (14.50,0.71)
        (15.50,0.96)
        (16.50,0.96)
        (17.50,0.99)
        (18.50,1.00)
        (19.50,1.26)
        (20.50,1.40)
        (21.50,1.81)
        (22.50,1.89)
        (23.50,3.23)
        (24.50,3.26)
        (25.50,3.40)
        (26.50,3.80)
        (27.50,4.13)
        (28.50,7.33)
    };
    \end{axis}
    \end{tikzpicture}
    \vspace{-1.5em}
    \caption{
        \ours-B/16 trained on LAION-CAT with 438M samples \vs CLIP-B/16 trained on LAION-2B with 2B samples. Both models trained by us with 32B total processed samples.
    }
    \label{fig:barlaion}
\end{figure}
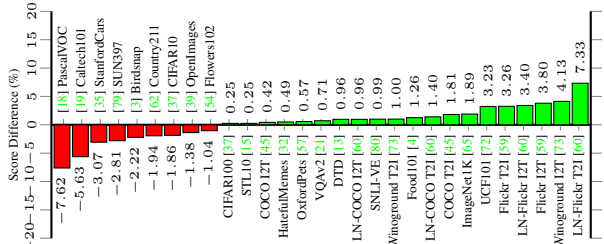

\begin{figure}
    \centering
    \vspace{-1.5em}
    \begin{tikzpicture}
    \begin{axis} [
        width=1.05\linewidth,
        height=4.6cm,
        font=\tiny,
        ybar,
        bar width=6pt,
        axis x line=center,
        axis y line=box,
        x axis line style={-,thick},
        y axis line style={thick},
        ytick align=outside,
        ytick={-45,-40,...,45},
        yticklabel style={rotate=90, yshift=-2pt},
        xtick={1},
        xticklabels={\cite{ffp04}~Caltech101},
        xticklabel style={yshift=2pt, anchor=west, rotate=90},
        extra x ticks={2,3,...,29},
        extra x tick style={xticklabel style={yshift=0pt, anchor=east}},
        extra x tick labels={SNLI-VE~\cite{xldk19},VQAv2~\cite{gks+17},PascalVOC~\cite{evw+10},HatefulMemes~\cite{kfm20},STL10~\cite{cnl11},Winoground T2I~\cite{tjb+22},LN-Flickr I2T~\cite{puc+20},Country211~\cite{rkh+21},LN-COCO I2T~\cite{puc+20},Flickr I2T~\cite{pwl+15},Winoground I2T~\cite{tjb+22},LN-COCO T2I~\cite{puc+20},OpenImages~\cite{kra+20},COCO T2I~\cite{lmb+14},LN-Flickr T2I~\cite{puc+20},COCO I2T~\cite{lmb+14},Food101~\cite{bgg14},Flickr T2I~\cite{pwl+15},SUN397~\cite{xeh+16},CIFAR10~\cite{kh09},ImageNet1K~\cite{rds+15},UCF101~\cite{szs12},DTD~\cite{cmk+14},OxfordPets~\cite{pvzj12},StanfordCars~\cite{ksdf13},Birdsnap~\cite{blw+14},Flowers102~\cite{nz08},CIFAR100~\cite{kh09},},
        xmin = 0.35,
        xmax = 29.65,
        ymin = -20,
        ymax = 20,
        ylabel = {Score Difference (\%)},
        ylabel near ticks,
        ylabel shift = -5pt,
        nodes near coords,
        nodes near coords align={horizontal},
        nodes near coords style={/pgf/number format/fixed,/pgf/number format/fixed zerofill,/pgf/number format/precision=2, rotate=90},
    ]
    \addplot[black,fill=red] coordinates {
        (1.50,-0.61)
    };
    \addplot[black,fill=green] coordinates {
        (1.50,1.37)
        (2.50,2.16)
        (3.50,2.42)
        (4.50,2.95)
        (5.50,3.00)
        (6.50,4.76)
        (7.50,5.50)
        (8.50,6.26)
        (9.50,6.70)
        (10.50,7.00)
        (11.50,7.13)
        (12.50,7.80)
        (13.50,9.00)
        (14.50,9.57)
        (15.50,9.71)
        (16.50,9.78)
        (17.50,9.90)
        (18.50,10.14)
        (19.50,10.57)
        (20.50,11.88)
        (21.50,13.21)
        (22.50,13.22)
        (23.50,13.46)
        (24.50,13.52)
        (25.50,15.32)
        (26.50,15.73)
        (27.50,16.83)
        (28.50,18.13)
    };
    \end{axis}

    \end{tikzpicture}
    \vspace{-1.5em}
    \caption{
        \ours-B/16 \vs CLIP-B/16. Both models trained by us on PMD with 63M images and 4B total processed samples.
    }
    \label{fig:barpmd}
    \vspace{-1em}
\end{figure}
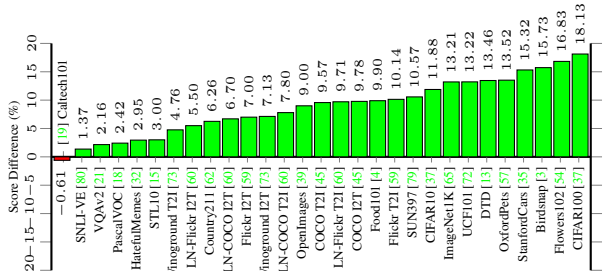

\vspace{-1em}
\paragraph{Robustness to distribution shift.}
We evaluate the ImageNet-related robustness on different datasets in Table~\ref{tab:ood}, for DiHT-B/16 \vs CLIP-B/16 baseline, trained by us on PMD and LAION datasets.
Our proposed approach improves on robustness over vanilla CLIP, in some cases by a significant margin, \eg ImageNet-A and ObjectNet.

\begin{table}[t]
\centering
\caption{
    ImageNet robustness performance on ViT-B/16 models.
}
\vspace{-1em}
\label{tab:ood}
\def\arraystretch{0.85}  
\setlength{\tabcolsep}{0.4mm}  
\def\cw{0.7cm}
\def\cww{0.7cm}
\def\hs{1mm}
\newcommand\cs[1]{#1}
\newcommand\css[1]{\scriptsize{#1}}
\newcommand\csss[1]{\tiny{#1}}
\centering
\footnotesize{
    \begin{tabular}{L{1cm}L{1.7cm}C{0.7cm}C{\cww}C{\cww}C{\cww}C{\cww}C{\cww}C{\cww}}
        \toprule
        \multirow{3}{*}{Method} & \multirow{3}{*}{Train Data} & \multirow{3}{*}{\#D} & \rotatebox[origin=b]{90}{IN~~~~~} & \rotatebox[origin=b]{90}{IN-V2~~} & \rotatebox[origin=b]{90}{IN-A~~~} & \rotatebox[origin=b]{90}{IN-R~~~} & \rotatebox[origin=b]{90}{IN-Sketch} & \rotatebox[origin=b]{90}{ObjectNet} \\
        \midrule
        CLIP & PMD & 63M & 54.6 & 47.9 & 35.5 & 56.3 & 30.8 & 38.5 \\
        DiHT & PMD & 63M & \textbf{67.8} & \textbf{61.5} & \textbf{54.3} & \textbf{74.4} & \textbf{44.7} & \textbf{52.8} \\
        \midrule
        CLIP & LAION-2B & 1.98B & 70.3 & 62.7 & 39.3 & 81.0 & 57.1 & 56.2 \\
        DiHT & LAION-CAT & 438M & \textbf{72.2} & \textbf{64.3} & \textbf{49.2} & \textbf{85.1} & \textbf{58.3} & \textbf{62.3} \\
        \bottomrule
    \end{tabular}
}
\vspace{-1.5em}
\end{table}

\subsection{Comparison with zero-shot state of the art}
\label{sec:exp_sota}

We compare our \ours models against state-of-the-art dual-encoder models in Table~\ref{tab:sota}.
Given that all models use different architectures, input image resolutions, training databases, and number of processed samples at training, we outline those details in the table, for easier comparison.

Our approach is most similar to CLIP~\cite{rkh+21} and OpenCLIP~\cite{iww+21}, and has same training complexity and inference complexity.
We outperform models with same architecture by substantial margins, even when our training dataset is much smaller.
Our best models \ours-L/14 and \ours-L/14@336 trained at higher $336$px resolution for additional 400M samples outperform models with significantly more complexity on popular text-image COCO and Flickr benchmarks. Compared to ALIGN~\cite{jyx+21} that has approximately twice the number of parameters compared to our \ours-L/14 model and is trained on 4x bigger data, we improve the performance substantially for all the retrieval benchmarks. Our model also performs better than FILIP~\cite{yhh+21} which utilizes token-wise similarity to compute the final alignment, thus noticeably increasing the training speed and memory cost. We also outperform Florence~\cite{lcc+21} on all $4$ retrieval benchmarks. Note that Florence~\cite{lcc+21} utilizes a more recent and powerful Swin-H Vision Transformer architecture~\cite{llc+21} with convolutional embeddings~\cite{wxc+21}, and a unified contrastive objective~\cite{ylz+22}.
Our proposed contributions are complementary to FILIP~\cite{yhh+21} and Florence~\cite{lcc+21}, and we believe additional gains can be achieved when combined.
Finally, LiT~\cite{zwm+22} and BASIC~\cite{pdgk21} first pre-train model on an large-scale image annotation dataset with cross-entropy before further training with contrastive loss on an image-text dataset.
Though this strategy results in state-of-the-art performance on ImageNet1K~\cite{rds+15} and image classification benchmarks, it has severe downsides on multi-modal tasks such as cross-modal retrieval.
Our ablation in Section~\ref{sec:exp_ablation} also confirms this issue.
On the other hand, our approach does not suffer from such negative effects.

\begin{table}[t]
\centering
\caption{
    Comparison with zero-shot state-of-the-art dual-encoder models.
    px: input image size; \#P: model size; \#D: training dataset size; \#S: total samples processed at training.
    We evaluate CLIP~\cite{rkh+21} and OpenCLIP~\cite{iww+21} using our codebase, other numbers are copied from respective papers.
    Grouped models (\eg, ViT-B/32) share same vision and language architecture as our model, following CLIP~\cite{rkh+21}, others have different architectures and we outline the vision one.
    $^*$FILIP uses token-wise similarity, which is more expensive than global-token similarity and requires adapting the architecture, hence we put it in ``Other''.
}
\vspace{-1em}
\label{tab:sota}
\def\arraystretch{0.85}  
\setlength{\tabcolsep}{0.4mm}  
\def\cw{0.7cm}
\def\cww{0.6cm}
\def\hs{1mm}
\newcommand\cs[1]{#1}
\newcommand\css[1]{\scriptsize{#1}}
\newcommand\csss[1]{\tiny{#1}}
\centering
\footnotesize{
    \begin{tabular}{L{1.8cm}C{0.6cm}C{\cw}C{\cw}C{\cw}C{\cww}C{\cww}C{\cww}C{\cww}C{\cww}}
        \toprule
        \multirow{2}{*}{Method} & \multirow{2}{*}{px} & \multirow{2}{*}{\#P} & \multirow{2}{*}{\#D} & \multirow{2}{*}{\#S} & \multirow{2}{*}{IN} & \multicolumn{2}{c}{COCO} & \multicolumn{2}{c}{Flickr} \\
        \cmidrule(l{\hs}r{\hs}){7-8}\cmidrule(l{\hs}r{\hs}){9-10}
        & & & & & & T2I & I2T & T2I & I2T \\
        \midrule
        \multicolumn{8}{l}{~~ViT-B/32} \\
        \cmidrule(l{\hs}r{\hs}){1-1}
        CLIP~\cite{rkh+21} & \cs{224} & \cs{151M} & \cs{400M} & \cs{12.8B} & 63.4 & 31.4 & 49.0 & 59.5 & 79.9 \\
        OpenCLIP~\cite{iww+21} & \cs{224} & \cs{151M} & \cs{400M} & \cs{12.8B} & 62.9 & 34.8 & 52.3 & 61.7 & 79.2 \\
        OpenCLIP~\cite{iww+21} & \cs{224} & \cs{151M} & \cs{2.3B} & \cs{34B} & 66.6 & 39.0 & 56.7 & 65.7 & 81.7 \\
        \ours & \cs{224} & \cs{151M} & \cs{438M} & \cs{16B} & 67.5 & 40.3 & 56.3 & 67.9 & 83.8 \\
        \ours & \cs{224} & \cs{151M} & \cs{438M} & \cs{32B} & {\bf 68.0} & {\bf 40.6} & {\bf 59.3} & {\bf 68.6} & {\bf 84.4} \\
        \midrule
        \multicolumn{8}{l}{~~ViT-B/16} \\
        \cmidrule(l{\hs}r{\hs}){1-1}
        CLIP~\cite{rkh+21} & \cs{224} & \cs{150M} & \cs{400M} & \cs{12.8B} & 68.4 & 33.7 & 51.3 & 63.3 & 81.9 \\
        OpenCLIP~\cite{iww+21} & \cs{224} & \cs{150M} & \cs{400M} & \cs{12.8B} & 67.1 & 37.8 & 55.4 & 65.2 & 84.1 \\
        OpenCLIP~\cite{iww+21} & \cs{240} & \cs{150M} & \cs{400M} & \cs{12.8B} & 69.2 & 40.5 & 57.8 & 67.7 & 85.3 \\
        \ours & \cs{224} & \cs{150M} & \cs{438M} & \cs{16B} & 71.9 & {\bf 43.7} & {\bf 62.0} & {\bf 73.2} & 89.5 \\
        \ours & \cs{224} & \cs{150M} & \cs{438M} & \cs{32B} & {\bf 72.2} & 43.3 & 60.3 & 72.9 & {\bf 89.8} \\
        \midrule
        \multicolumn{8}{l}{~~ViT-L/14} \\
        \cmidrule(l{\hs}r{\hs}){1-1}
        CLIP~\cite{rkh+21} & \cs{224} & \cs{428M} & \cs{400M} & \cs{12.8B} & 75.6 & 36.5 & 54.9 & 66.1 & 84.5 \\
        CLIP~\cite{rkh+21} & \cs{336} & \cs{428M} & \cs{400M} & \cs{13.2B} & 76.6 & 37.7 & 57.1 & 68.6 & 86.6 \\
        OpenCLIP~\cite{iww+21} & \cs{224} & \cs{428M} & \cs{400M} & \cs{12.8B} & 72.8 & 42.1 & 60.1 & 70.4 & 86.8 \\
        OpenCLIP~\cite{iww+21} & \cs{224} & \cs{428M} & \cs{2.3B} & \cs{32B} & 75.2 & 46.2 & 64.3 & 75.4 & 90.4 \\
        \ours & \cs{224} & \cs{428M} & \cs{438M} & \cs{16B} & 77.0 & 48.0 & 65.1 & 76.7 & {\bf 92.0} \\
        \ours & \cs{336} & \cs{428M} & \cs{438M} & \cs{16.4B} & {\bf 77.9} & {\bf 49.3} & {\bf 65.3} & {\bf 78.2} & 91.1\\
        \midrule
        \multicolumn{8}{l}{~~Other} \\
        \cmidrule(l{\hs}r{\hs}){1-1}
        ALIGN~\cite{jyx+21} \csss{EfficientNet-L2} & \cs{289} & \cs{820M} & \cs{1.8B} & \cs{19.7B} & 76.4 & 45.6 & 58.6 & 75.7 & 88.6 \\
        FILIP~\cite{yhh+21}$^*$ \csss{ViT-L/14} & \cs{224} & \cs{428M} & \cs{340M} & \cs{10.2B} & 77.1 & 45.9 & 61.3 & 75.0 & 89.8 \\
        OpenCLIP~\cite{iww+21} \csss{ViT-H/14} & \cs{224} & \cs{986M} & \cs{2.3B} & \cs{32B} & 77.9 & 49.0 & 67.5 & 76.8 & 91.3 \\
        Florence~\cite{lcc+21} \csss{CoSwin-H} & \cs{384} & \cs{893M} & \cs{900M} & \cs{31B} & 83.7 & 47.2 & 64.7 & 76.7 & 90.9 \\
        LiT~\cite{zwm+22} ~~~~~~~~~~\csss{ViT-g/14} & \cs{288} & \cs{2.0B} & \cs{3.6B} & \cs{18.2B} & 85.2 & 41.9 & 59.3 & --- & --- \\
        BASIC~\cite{pdgk21} \csss{CoAtNet-7} & \cs{224} & \cs{3.1B} & \cs{6.6B} & \cs{32.8B} & 85.7 & --- & --- & --- & --- \\
        \bottomrule
    \end{tabular}
}
\end{table}

\subsection{Few-shot linear probing}
\label{sec:exp_few}
The ideal scenario for leveraging zero-shot recognition models is to warm start the task without training data and then improve the performance (by training a linear probe) via few-shot learning as more and more data is seen. However, in practice, few-shot models  perform significantly worse than zero-shot models in the {\em low-data} regime.

We present an alternate approach to do few-shot classification with prompt-based initialization. The key idea of our approach is to initialize the classifier with the {\em zero-shot} text prompts for each class, but to also ensure that the final weights do not drift much from the prompt using {\em projected gradient descent (PGD)}~\cite{b04}. While few-shot models have been initialized with prompt priors in the past with naive $L_2$ penalties for weight to prevent catastrophic forgetting~\cite{kprv+17}, these approaches do not improve performance and the model simply ignores the supervision. In contrast, for any target dataset $\cD_{\text{target}} = \{(\bx_i, y_i)_{i=1}^n\}$, where $\bx_i = \bphi_{\text{image}}(I_i)$ denotes the image features from the {\em trained} image tower, we solve the following optimization problem, for some $\delta, \delta_b > 0$:
\begin{equation*}
\min_{\lVert \bW \rVert_2 \leq \delta, \lVert \boldb \rVert_2 \leq \delta_b}\ \sum_{i=1}^n\cL_{\text{CE}}\left(y_i, \bx_i^\top\left(\bW+\bW_\text{0}\right) + \boldb\right). 
\end{equation*}
Here $\bW_0 \in \bbR^{d\times n_c}$ denotes the prompt initialization from the text encoder. To optimize the objective, one can use projected gradient descent~\cite{b04}.  We observe that our approach is able to bridge the gap between zero-shot and 1-shot classification, a common issue in prior linear probe evaluations. 

Figure~\ref{fig:few} presents the full summary of results on the ImageNet1K~\cite{rds+15} $k$-shot classification task. Hyperparameters $\delta$ and $\delta_b$ for our approach, and weight decay for the baseline approach of training linear probes from scratch are found using grid search. Note that compared to the baseline, our method performs substantially better at very low values of $k$ and maintains the performance continuum from zero-shot to 1-shot, and so on. At large $k$ values, both approaches perform similarly, since there are sufficient data samples to render the zero-shot initialization ineffective. To further showcase the strength of our approach, we also compare our performance with linear probes trained on powerful SWAG~\cite{sga+22} models that are especially suited for this task. Note that our approach outperforms the much larger SWAG ViT-H/14 model up to 25-shot classification. We would like to emphasize that this albeit straightforward approach is one of the first to resolve this discontinuity problem between zero-shot and few-shot learning.

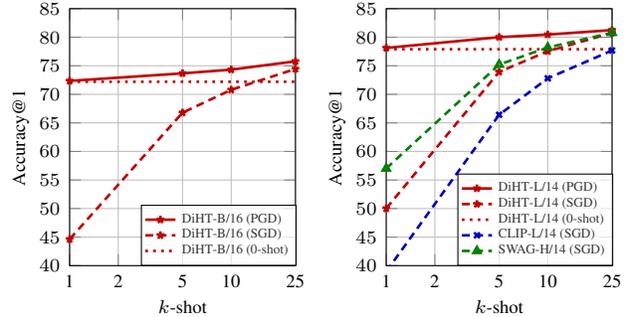
\begin{figure}[t]
\centering
\pgfplotstableread{
    samples     l_in    l_coco_t2i  l_coco_i2t  l_flckr_t2i     l_flckr_i2t     l_filt_in   l_filt_coco_t2i     l_filt_coco_i2t     l_filt_flckr_t2i    l_filt_flckr_i2t
    4           0.608   0.337       0.521       0.593           0.777           0.615       0.376               0.559               0.665               0.832
    8           0.634   0.357       0.534       0.608           0.794           0.633       0.388               0.564               0.680               0.832
    16          0.645   0.368       0.551       0.628           0.812           0.644       0.393               0.560               0.681               0.854
    32          0.651   0.373       0.553       0.632           0.804           0.648       0.400               0.574               0.684               0.843
}{\filtering}

\pgfplotstableread{
    id  clip_in     clip_coco_t2i   clip_flickr_t2i     ours_in     ours_coco_t2i   ours_flickr_t2i
    1   0.634       0.314           0.595               0.680       0.406           0.686
    2   0.684       0.337           0.633               0.722       0.433           0.729
    3   0.766       0.377           0.686               0.779       0.493           0.782
}{\clipvsours}

\pgfplotstableread{
    k       ours_b16_pgd    ours_b16_sgd    ours_l14_336_pgd    ours_l14_336_sgd    clip_l14_336_sgd    swag_h14_sgd
    0       0.7220          0.7220          0.7790              0.7790              0.7660              nan
    1       0.7235          0.4463          0.7812              0.5001              0.3896              0.5701
    5       0.7367          0.6678          0.8001              0.7393              0.6642              0.7521
    10      0.7432          0.7079          0.8046              0.7753              0.7280              0.7820
    25      0.7578          0.7444          0.8125              0.8076              0.7768              0.8078
    50      0.7777          0.7658          nan                 nan                 nan                 nan
    100     0.7897          0.7821          nan                 nan                 nan                 nan
}{\fewshot}
\hspace{-1em}
\begin{tikzpicture}
    \begin{semilogxaxis}[%
        ylabel near ticks, ylabel shift = -3pt, yticklabel pos=left,
        xlabel near ticks, xlabel shift = -2pt,
        font=\scriptsize,
        width=0.55\linewidth,
        height=0.6\linewidth,
        title style = {yshift = -5pt},
        xlabel={$k$-shot},
        ylabel={Accuracy@1},
        xlabel style  = {yshift = 0pt},
        legend image code/.code={ 
            \draw[mark repeat=2,mark phase=2]
            plot coordinates {
                (0cm,0cm)
                (0.2cm,0cm)  
                (0.4cm,0cm)  
            };%
        },
        legend pos=south east,
        legend style={cells={anchor=west}, font =\tiny, inner sep=1pt, fill opacity=0.8, row sep=-3pt, xshift=0.3em, yshift=-0.3em},
        ymin = 40,
        ymax = 85,
        xmin = 1,
        xmax = 25,
        grid=both,
        xtick={1,2,5,10,25,50,100},
        xticklabels={1,2,5,10,25,50,100},
        ytick={40,45,...,100},
        tick label style ={font=\scriptsize},
        xticklabel style={font=\scriptsize,align=left},
    ]  
        \addplot[color=darkred, solid, mark=star, mark options={solid}, mark size=1.5, line width=1] table[x=k, y expr={100*\thisrow{ours_b16_pgd}}] \fewshot; \leg{\ours-B/16 (PGD)};
        \addplot[color=darkred, densely dashed, mark=star, mark options={solid}, mark size=1.5, line width=1] table[x=k, y expr={100*\thisrow{ours_b16_sgd}}] \fewshot; \leg{\ours-B/16 (SGD)};
        \addplot[color=darkred, dotted, mark=none, mark size=1.5, mark options={solid}, line width=1] table[x=k, y expr={100*0.7220}] \fewshot; \leg{\ours-B/16 (0-shot)};
    \end{semilogxaxis}
\end{tikzpicture}
\begin{tikzpicture}
    \begin{semilogxaxis}[%
        ylabel near ticks, ylabel shift = -3pt, yticklabel pos=left,
        xlabel near ticks, xlabel shift = -2pt,
        font=\scriptsize,
        width=0.55\linewidth,
        height=0.6\linewidth,
        title style = {yshift = -5pt},
        xlabel={$k$-shot},
        ylabel={Accuracy@1},
        xlabel style  = {yshift = 0pt},
        legend image code/.code={ 
            \draw[mark repeat=2,mark phase=2]
            plot coordinates {
                (0cm,0cm)
                (0.2cm,0cm)  
                (0.4cm,0cm)  
            };%
        },
        legend pos=south east,
        legend style={cells={anchor=west}, font =\tiny, inner sep=1pt, fill opacity=0.8, row sep=-3pt, xshift=0.3em, yshift=-0.3em},
        ymin = 40,
        ymax = 85,
        xmin = 1,
        xmax = 25,
        grid=both,
        xtick={1,2,5,10,25,50,100},
        xticklabels={1,2,5,10,25,50,100},
        ytick={40,45,...,100},
        tick label style ={font=\scriptsize},
        xticklabel style={font=\scriptsize,align=left},
    ]  
        \addplot[color=darkred, solid, mark=star, mark options={solid}, mark size=1.5, line width=1] table[x=k, y expr={100*\thisrow{ours_l14_336_pgd}}] \fewshot; \leg{\ours-L/14 (PGD)};
        \addplot[color=darkred, densely dashed, mark=star, mark options={solid}, mark size=1.5, line width=1] table[x=k, y expr={100*\thisrow{ours_l14_336_sgd}}] \fewshot; \leg{\ours-L/14 (SGD)};
        \addplot[color=darkred, dotted, mark=none, mark size=1.5, mark options={solid}, line width=1] table[x=k, y expr={100*0.779}] \fewshot; \leg{\ours-L/14 (0-shot)};
        \addplot[color=darkblue, densely dashed, mark=x, mark options={solid}, mark size=1.5, line width=1] table[x=k, y expr={100*\thisrow{clip_l14_336_sgd}}] \fewshot; \leg{CLIP-L/14 (SGD)};
        \addplot[color=darkgreen, densely dashed, mark=triangle, mark options={solid}, mark size=1.5, line width=1] table[x=k, y expr={100*\thisrow{swag_h14_sgd}}] \fewshot; \leg{SWAG-H/14 (SGD)};
    \end{semilogxaxis}
\end{tikzpicture}
\vspace{-1em}
\caption{
    $k$-shot linear probing performance on ImageNet1K.
    \label{fig:few}
}
\vspace{-1em}
\end{figure}

\section{Conclusion and future work}
\label{sec:conclusion}

In this paper, we demonstrate that with careful dataset filtering and simple but effective modeling changes, it is possible to achieve substantial improvements in zero-shot performance on retrieval and classification tasks through large-scale pre-training. Our CAT filtering approach can be applied generically to any large-scale dataset for improved performance with smaller training schedules. Moreover, our concept distillation approach presents a compute and storage efficient way of leveraging very large capacity pre-trained image models for multimodal training. Finally, our simple projected gradient approach covers the crucial performance gap between zero-shot and few-shot learning.

In future, we would like to extend our approach  to multi-modal encoder/decoder~\cite{adl+22,ywv+22,cwc+22,lsg+21,ydt+22} architectures that although expensive, have better zero-shot performance compared to dual encoders. We also observe that benefits of our hard-negatives loss are less on noisier LAION dataset compared to PMD. It would be interesting to explore how to make it more effective in these very noisy settings. We hope that our improvements and extensive large-scale ablations will further advance the vision-language research.

{
\small
\bibliographystyle{ieee_fullname}
\bibliography{main}
}

\clearpage
\appendix
\counterwithin{figure}{section}
\counterwithin{table}{section}
\counterwithin{equation}{section}

\section{Appendix}

\subsection{Method details}

\subsubsection{Semantic parser}
To enable a rich complexity and semantic filtering, we built a fast custom semantic parser that converts a given textual caption to a semantic graph similar to the one in Visual Genome\cite{kzg+17}. In particular, we extract objects, their parts, their attributes, and the actions that they are involved in (see Figure \ref{fig:object-data-parser} for example). 
The parser is built on top of the English language dependency parser from Spacy\cite{spacy+2022} combined with multiple rules to infer common object relations. The aim of the parser is high speed with high precision of common object relations such as `has\_attribute` and `has\_part` and basic `action` support.
Below, we describe the structured relations that we extract from natural language text.

\vspace{1em}\noindent We support the following semantic relations:

\vspace{-0.5em}\paragraph{Object (\_obj).}
We extract objects that are supposedly presented in an image. 
We consider nouns that are not attributes of another noun (not part of a noun phrase).
\Eg in \textit{birthday cake} and \textit{baby stroller}, the nouns \textit{cake} and \textit{stroller} are parsed as objects, and the nouns \textit{birthday} and \textit{baby} are considered attributes. We do not consider proper nouns.

\vspace{-0.5em}\paragraph{Attribute (has\_attr).}
Denotes attributes that characterize an object or another attribute. For example, \textit{dark green}, would result in a fact \textit{green - has\_attr - dark}, and \textit{yellow candles} results in \textit{candles - has\_attr - yellow}.

\vspace{-0.5em}\paragraph{Part (has\_part).}
Characterizes a visual part of an object. \Eg \textit{cake with 21 yellow candles} would result in a part fact \textit{cake - has\_part - candles}.

\vspace{-0.5em}\paragraph{Action (\_act).}
Verbs that do not entail attributes or parts (\eg forms of \textit{be}, \textit{looks}, \textit{seems}, and \textit{have} are excluded) are considered actions. For actions, we also parse the subject and object arguments.

\vspace{-0.5em}\paragraph{Subject of an action (act\_has\_subj, is\_act\_subj).}
We use the act\_has\_subj and is\_act\_subj relation to represent arguments (nouns) that are the subject of an action.
\Eg for the text \textit{a person is eating an apple}, we add the object-centric and corresponding action-centric symmetric facts: \textit{person - is\_subj\_act - eating} and \textit{eating act\_has\_subj person}. 

\vspace{-0.5em}\paragraph{Object of an action (act\_has\_obj, is\_act\_obj).}
We also include the relations that specify the object arguments of an action.
\Eg for the text \textit{a person is eating an apple}, we add the object-centric and corresponding action-centric symmetric facts: \textit{apple - is\_obj\_act - eating} and \textit{eating act\_has\_obj apple}.

\vspace{1em}\noindent We recognize the following limitations of ours approach:

\paragraph{Semantic attributes.}  In this work, we focus on object-centric visual and action characteristics and we do not process spatial relations ( X next to Y) or additional action arguments (read a book *in* the library). Spatial relations and additional arguments of verbs usually involve more complex semantic reasoning and require more robust approaches and task-specific models such as one trained on Semantic Role Labeling which are usually compute-heavy. We leave these for future work. 

\paragraph{Dependency parser errors.} In the current version of the parser, we also parse potential attributes as actions, which are not likely to be always visual.
\Eg In the phrase “running person”, running is an action and an attribute, and we parse them as such. However, sometimes the underlying parser would also parse attributes in phrases such as “striped mug” as verbs, where we process the attribute “striped” as both an attribute and an action (without arguments).

\subsubsection{Concept distillation}
The teacher model is built by training linear classifiers - which predict objects and attributes - on top of a frozen SWAG~\cite{sga+22} backbone.
SWAG is trained in a weakly-supervised manner by predicting hashtags from Instagram images.
We use the publicly available weights, and adopt a training procedure that is similar to the one from SWAG for learning the linear classifiers.
The procedure for training the object classifier is as follows.
First, we parse the captions to extract nouns. Next, we canonicalize the nouns via WordNet~\cite{m+98} synsets and remove ones which occur less than 250 times in the dataset.
The resulting vocabulary contains $\sim$10K unique synsets.
Finally, we optimize the linear layer's weights through a cross-entropy loss.
Each entry in the target distribution of the cross-entropy is either $1 / K$ or $0$ depending on whether the corresponding synset is present or not, where $K$ is the number of synsets for that image.
We apply inverse square-root resampling of images to upsample the tail classes following~\cite{sga+22}.
The target length of the dataset is set to 50 million samples during resampling .
We train the linear layer using SGD with momentum 0.9 and weight decay 1e-4.
The learning rate is set following the linear scaling rule: lr$=$0.001$\cdot$$\frac{bs}{256}$.
To speedup training, we use 64 GPUs with batch size of 256 per GPU.
The attribute classifiers are build in a similar way, but the WordNet adjective synsets require additional filtering to remove non-visual attributes, \eg, \emph{claustrophobic}, \emph{experienced}.
Following~\cite{rsg+21}, we select the attributes based on their \emph{sharedness} and \emph{visualness}.
We rank the attributes based on the aforementioned scores, and keep $\sim$1200 attributes.

\subsection{Training details}
For our model architecture, we closely follow CLIP by Radford~\etal~\cite{rkh+21}.
We utilize Vision Transformers (ViT)~\cite{dbk+20} for images and Text Transformers~\cite{vsp+17} for captions.
We experiment with 3 different architectures, denoted as B/32, B/16, and L/14, where 32, 16, and 14 denote the input image patch size.
Other architecture scaling parameters are in Table~\ref{tab:arch_hypers}.
For distillation and fine-tuning experiments, we utilize the public SWAG-ViT models~\cite{sga+22}, pre-trained with weak supervision from hashtags.

\begin{table}[t]
\centering
\caption{\ours architecture hyperparameters.}
\label{tab:arch_hypers}
\vspace{-1em}
\def\arraystretch{0.8}  
\setlength{\tabcolsep}{1mm}  
\def\cw{0.8cm}
\def\hs{1mm}
\footnotesize{
    \begin{tabular}{L{1cm}C{0.7cm}C{\cw}C{\cw}C{\cw}C{\cw}C{\cw}C{\cw}}
    \toprule
    \multirow{2}{*}{Model} & \multirow{2}{*}{Dim} & \multicolumn{3}{c}{Vision} & \multicolumn{3}{c}{Language} \\
    \cmidrule(l{\hs}r{\hs}){3-5}\cmidrule(l{\hs}r{\hs}){6-8}
    & & layers & width & heads & layers & width & heads \\
    \midrule
    B/32 &  512 & 12 &  768 & 12 & 12 &  512 &  8 \\
    B/16 &  512 & 12 &  768 & 12 & 12 &  512 &  8 \\
    L/14 &  768 & 24 & 1024 & 16 & 12 &  768 & 12 \\
    \bottomrule
    \end{tabular}
}
\end{table}

We use the Adam~\cite{kb15} optimizer with a decoupled weight decay~\cite{lh19} and a cosine learning rate schedule~\cite{lh17}.
Input image size is 224$\times$224 pixels, for pre-training runs.
All hyperparameters are presented in Table~\ref{tab:common_hypers}.
They are selected by training on a small scale setup, and reused for other experiments.
For \objects and attributes classifiers in concept distillation (CD), we found that scaling the learning rate by 10.0 and weight decay by 0.01 gave better results.

\begin{table}[t]
\centering
\caption{\ours common hyperparameters.}
\label{tab:common_hypers}
\vspace{-1em}
\def\arraystretch{0.9}  
\setlength{\tabcolsep}{1mm}  
\def\hs{1mm}
\footnotesize{
    \begin{tabular}{L{3cm}C{1.5cm}C{1.5cm}C{1.5cm}}
    \toprule
    \textbf{Shared} \\
    \midrule
    Learning rate (LR) & \multicolumn{3}{c}{1e-3} \\
    Warm-up & \multicolumn{3}{c}{1\%} \\
    Vocabulary size & \multicolumn{3}{c}{49408} \\
    Temperature (init, max) & \multicolumn{3}{c}{($\frac{1}{0.07}$, 100.0)} \\
    Adam ($\beta_1$, $\beta_2$) & \multicolumn{3}{c}{(0.9, 0.98)} \\
    Adam $\epsilon$ & \multicolumn{3}{c}{1e-6} \\
    High resolution LR & \multicolumn{3}{c}{1e-4} \\
    \midrule
    \textbf{Dataset specific} & \multicolumn{2}{c}{LAION} & PMD \\
    \midrule
    CD learning rate scale & \multicolumn{2}{c}{10.0} & 1.0 \\
    CD weight decay scale & \multicolumn{2}{c}{0.01} & 1.0 \\
    HN-NCE $\alpha$ & \multicolumn{2}{c}{1.0} & 0.999 \\
    HN-NCE $\beta$ & \multicolumn{2}{c}{0.25} & 0.5 \\
    \midrule
     & \multicolumn{2}{c}{LAION} & PMD \\
     \cmidrule(l{\hs}r{\hs}){2-3}\cmidrule(l{\hs}r{\hs}){4-4}
    \textbf{Model specific} & L/14 & B/16,B/32 & B/16,B/32 \\
    \midrule
    Batch size & 98304 & 49152 & 32768 \\
    Weight decay & 0.2 & 0.1 & 0.1 \\
    \bottomrule
    \end{tabular}
    \vspace{-1em}
}
\end{table}

We pre-train the models on 4B, 8B, 16B, or 32B processed samples, depending on the experiment.
For L/14 we train at a higher 336px resolution for additional 400M samples, denoting this models as L/14@336.
We trained L/14 for 6 days on 512 A100 GPUs with 16B processed samples for a total of $7.4 \times 10^4$ GPU hours.

To accelerate training and save memory, we use mixed-precision training~\cite{mna+18}.
For L/14 we use grad checkpointing~\cite{cxz+16} and BFLOAT16~\cite{bf16+22,kmm+19} format, all the other models are trained using FP16~\cite{mna+18} format.
Contrastive loss is computed on the local subset of the pairwise similarities~\cite{rkh+21}.

\begin{figure*}[t]
\centering
\pgfplotstableread{
    samples     l_in    l_coco_t2i  l_coco_i2t  l_flckr_t2i     l_flckr_i2t     l_filt_in   l_filt_coco_t2i     l_filt_coco_i2t     l_filt_flckr_t2i    l_filt_flckr_i2t
    4           0.608   0.337       0.521       0.593           0.777           0.615       0.376               0.559               0.665               0.832
    8           0.634   0.357       0.534       0.608           0.794           0.633       0.388               0.564               0.680               0.832
    16          0.645   0.368       0.551       0.628           0.812           0.644       0.393               0.560               0.681               0.854
    32          0.651   0.373       0.553       0.632           0.804           0.648       0.400               0.574               0.684               0.843
}{\filtering}

\pgfplotstableread{
    id  clip_in     clip_coco_t2i   clip_flickr_t2i     ours_in     ours_coco_t2i   ours_flickr_t2i
    1   0.634       0.314           0.595               0.680       0.406           0.686
    2   0.684       0.337           0.633               0.722       0.433           0.729
    3   0.766       0.377           0.686               0.779       0.493           0.782
}{\clipvsours}

\pgfplotstableread{
    k       ours_b16_pgd    ours_b16_sgd    ours_l14_336_pgd    ours_l14_336_sgd    clip_l14_336_sgd    swag_h14_sgd
    0       0.7220          0.7220          0.7790              0.7790              0.7660              nan
    1       0.7235          0.4463          0.7812              0.5001              0.3896              0.5701
    5       0.7367          0.6678          0.8001              0.7393              0.6642              0.7521
    10      0.7432          0.7079          0.8046              0.7753              0.7280              0.7820
    25      0.7578          0.7444          0.8125              0.8076              0.7768              0.8078
    50      0.7777          0.7658          nan                 nan                 nan                 nan
    100     0.7897          0.7821          nan                 nan                 nan                 nan
}{\fewshot}
\begin{tabular}{ccccc}
%
\begin{tikzpicture}
    \begin{semilogxaxis}[%
        ylabel near ticks, ylabel shift = -2pt, yticklabel pos=left,
        xlabel near ticks,
        font=\scriptsize,
        width=0.22\linewidth,
        height=0.25\linewidth,
        title={ImageNet-1K},
        title style = {yshift = -5pt},
        xlabel={Num Samples},
        ylabel={Accuracy@1},
        xlabel style  = {yshift = 0pt},
        legend pos=south east,
        legend style={cells={anchor=west}, font =\tiny, fill opacity=0.8, row sep=-2.5pt, xshift=0.15em, yshift=-0.2em},
        ymin = 58,
        ymax = 68,
        xmin = 3.6,
        xmax = 36,
        grid=both,
        xtick={4,8,16,32},
        xticklabels={4B, 8B, 16B, 32B},
        ytick={1, 2, ..., 100},
        tick label style ={font=\scriptsize},
    ]  
        \addplot[color=darkblue, solid, mark=x, mark size=1.5, line width=1] table[x=samples, y expr={100*\thisrow{l_in}}] \filtering; \leg{LAION-2B};
        \addplot[color=darkred, solid, mark=star, mark size=1.5, line width=1] table[x=samples, y expr={100*\thisrow{l_filt_in}}] \filtering; \leg{LAION-CAT};
    \end{semilogxaxis}
\end{tikzpicture}
\hspace{-1.25em}
&
\begin{tikzpicture}
    \begin{semilogxaxis}[%
        ylabel near ticks, ylabel shift = -2pt, yticklabel pos=left,
        xlabel near ticks,
        font=\scriptsize,
        width=0.22\linewidth,
        height=0.25\linewidth,
        title={COCO (T2I)},
        title style = {yshift = -5pt},
        xlabel={Num Samples},
        ylabel={Recall@1},
        xlabel style  = {yshift = 0pt},
        legend pos=south east,
        legend style={cells={anchor=west}, font =\tiny, fill opacity=0.8, row sep=-2.5pt},
        ymin = 32,
        ymax = 42,
        xmin = 3.6,
        xmax = 36,
        grid=both,
        xtick={4,8,16,32},
        xticklabels={4B, 8B, 16B, 32B},
        ytick={1, 2, ..., 100},
        tick label style ={font=\scriptsize},
    ]  
        \addplot[color=darkblue, solid, mark=x, mark size=1.5, line width=1] table[x=samples, y expr={100*\thisrow{l_coco_t2i}}] \filtering; 
        \addplot[color=darkred, solid, mark=x, mark size=1.5, line width=1] table[x=samples, y expr={100*\thisrow{l_filt_coco_t2i}}] \filtering; 
    \end{semilogxaxis}
\end{tikzpicture}
\hspace{-1.25em}
&
\begin{tikzpicture}
    \begin{semilogxaxis}[%
        ylabel near ticks, ylabel shift = -2pt, yticklabel pos=left,
        xlabel near ticks,
        font=\scriptsize,
        width=0.22\linewidth,
        height=0.25\linewidth,
        title={COCO (I2T)},
        title style = {yshift = -5pt},
        xlabel={Num Samples},
        ylabel={Recall@1},
        xlabel style  = {yshift = 0pt},
        legend pos=south east,
        legend style={cells={anchor=west}, font =\tiny, fill opacity=0.8, row sep=-2.5pt},
        ymin = 50,
        ymax = 60,
        xmin = 3.6,
        xmax = 36,
        grid=both,
        xtick={4,8,16,32},
        xticklabels={4B, 8B, 16B, 32B},
        ytick={1, 2, ..., 100},
        tick label style ={font=\scriptsize},
    ]  
        \addplot[color=darkblue, solid, mark=x, mark size=1.5, line width=1] table[x=samples, y expr={100*\thisrow{l_coco_i2t}}] \filtering; 
        \addplot[color=darkred, solid, mark=x, mark size=1.5, line width=1] table[x=samples, y expr={100*\thisrow{l_filt_coco_i2t}}] \filtering; 
    \end{semilogxaxis}
\end{tikzpicture}
\hspace{-1.25em}
&
\begin{tikzpicture}
    \begin{semilogxaxis}[%
        ylabel near ticks, ylabel shift = -2pt, yticklabel pos=left,
        xlabel near ticks,
        font=\scriptsize,
        width=0.22\linewidth,
        height=0.25\linewidth,
        title={Flickr (T2I)},
        title style = {yshift = -5pt},
        xlabel={Num Samples},
        ylabel={Recall@1},
        xlabel style  = {yshift = 0pt},
        legend pos=south east,
        legend style={cells={anchor=west}, font =\tiny, fill opacity=0.8, row sep=-2.5pt},
        ymin = 59,
        ymax = 69,
        xmin = 3.6,
        xmax = 36,
        grid=both,
        xtick={4,8,16,32},
        xticklabels={4B, 8B, 16B, 32B},
        ytick={1, 2, ..., 100},
        tick label style ={font=\scriptsize},
    ]  
        \addplot[color=darkblue, solid, mark=x, mark size=1.5, line width=1] table[x=samples, y expr={100*\thisrow{l_flckr_t2i}}] \filtering; 
        \addplot[color=darkred, solid, mark=x, mark size=1.5, line width=1] table[x=samples, y expr={100*\thisrow{l_filt_flckr_t2i}}] \filtering; 
    \end{semilogxaxis}
\end{tikzpicture}
\hspace{-1.25em}
&
\begin{tikzpicture}
    \begin{semilogxaxis}[%
        ylabel near ticks, ylabel shift = -2pt, yticklabel pos=left,
        xlabel near ticks,
        font=\scriptsize,
        width=0.22\linewidth,
        height=0.25\linewidth,
        title={Flickr (I2T)},
        title style = {yshift = -5pt},
        xlabel={Num Samples},
        ylabel={Recall@1},
        xlabel style  = {yshift = 0pt},
        legend pos=south east,
        legend style={cells={anchor=west}, font =\tiny, fill opacity=0.8, row sep=-2.5pt},
        ymin = 76,
        ymax = 86,
        xmin = 3.6,
        xmax = 36,
        grid=both,
        xtick={4,8,16,32},
        xticklabels={4B, 8B, 16B, 32B},
        ytick={1, 2, ..., 100},
        tick label style ={font=\scriptsize},
    ]  
        \addplot[color=darkblue, solid, mark=x, mark size=1.5, line width=1] table[x=samples, y expr={100*\thisrow{l_flckr_i2t}}] \filtering; 
        \addplot[color=darkred, solid, mark=x, mark size=1.5, line width=1]table[x=samples, y expr={100*\thisrow{l_filt_flckr_i2t}}] \filtering; 
    \end{semilogxaxis}
\end{tikzpicture}
\end{tabular}
\vspace{-1em}
\caption{
    Evaluating effect of using our LAION-CAT subset filtered on complexity (C), actions (A), and text spotting (T).
    Evaluation performed on ViT-B/32 architecture trained for a varying number of processed samples.
\label{fig:fitering}
}
\end{figure*}
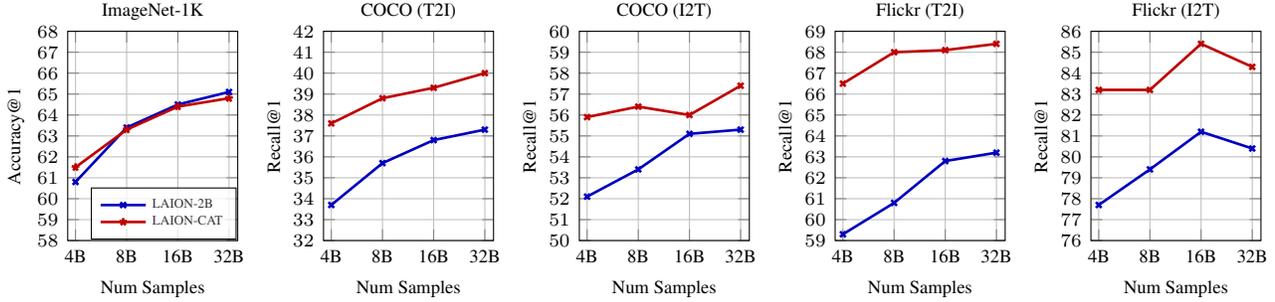

\subsection{Evaluation details}

We evaluate our models on a zero-shot benchmark of 24 datasets:
(i) \textbf{17 image classification}: Birdsnap~\cite{blw+14}, CIFAR10~\cite{kh09}, CIFAR100~\cite{kh09}, Caltech101~\cite{ffp04}, Country211~\cite{rkh+21}, DTD~\cite{cmk+14}, Flowers102~\cite{nz08}, Food101~\cite{bgg14}, ImageNet1K~\cite{rds+15}, OxfordPets~\cite{pvzj12}, STL10~\cite{cnl11}, SUN397~\cite{xeh+16}, StanfordCars~\cite{ksdf13}, UCF101~\cite{szs12}, HatefulMemes~\cite{kfm20}, PascalVOC2007~\cite{evw+10}, OpenImages~\cite{kra+20}; 
(ii) \textbf{5 cross-modal retrieval} (text-to-image T2I, image-to-text I2T): COCO~\cite{lmb+14}, Flickr~\cite{pwl+15}, LN-COCO~\cite{puc+20}, LN-Flickr~\cite{puc+20}, Winoground~\cite{tjb+22}; 
(iii) \textbf{2 visual question answering}: SNLI-VE~\cite{xldk19}, VQAv2~\cite{gks+17}.
Note that, cross-modal retrieval datasets have 2 tasks (T2I and I2T), so in total we evaluate across 29 tasks.

We follow zero-shot CLIP benchmark\footnote{\href{https://github.com/LAION-AI/CLIP_benchmark}{github.com/LAION-AI/CLIP\_benchmark}} implementation for most of the datasets, and implement the ones that are missing.
For most image classification tasks we compute Accuracy@1, except HatefulMemes where we compute AUROC because it is binary classification, OpenImages where we compute FlatHit@1 following~\cite{wcz+21}, and PascalVOC2007 where we compute mean average precision (mAP) because it is multi-label classification.
We use the same prompt ensembling method as CLIP~\cite{rkh+21} to improve zero-shot image classification.
For cross-modal retrieval (T2I and I2T), we compute Recall@1.
For COCO and Flickr we apply a simple prompt pretext ``a photo of \{\texttt{caption}\}'', for LN-COCO, LN-Flickr, and Winoground no prompt is applied.
We cast visual question answering (VQA) as binary prediction task and compute AP on the cosine similarity between an image and a text (a hypothesis or a question).
For SNLI-VE, we take a subset which has agreement among annotators, we use ``entailement'' and ``contradiction'' as binary classes, and drop the ``neutral" class.
For VQAv2, we take the subset with yes/no questions.
No prompt is applied for SNLI-VE and VQAv2.

\subsection{Additional ablations}

\paragraph{Effect of dataset filtering.}
In Figure~\ref{fig:fitering} we observe that gains from our proposed complexity, action, and text-spotting (CAT) dataset filtering hold as we train for longer training schedules.
We ran small scale experiments with several complexity filters (see Table \ref{tab:filtering_numbers}) and we found that CAT with minimum complexity C1 performed the best.

\begin{table}[t]
\centering
\small
\caption{
    Number of examples after filtering with different filters.
}
\label{tab:filtering_numbers}
\vspace{-1em}
\def\arraystretch{1.0}  
\setlength{\tabcolsep}{0.5mm}
\def\cw{0.6cm}
\def\hs{1mm}
\small{
    \begin{tabular}{C{\cw}C{\cw}C{\cw}C{\cw}C{\cw}C{\cw}R{2cm}R{1.2cm}}
        \toprule
        \multicolumn{6}{c}{Filter} &  \multirow{2}{*}{\# examples} &  \multirow{2}{*}{\% of full} \\
        \cmidrule(l{\hs}r{\hs}){1-6}
        \cite{sph+22} & C0 & C1 & C2 & A & T & & \\
        \midrule
        & & & & & & 2,121,505,329 & 100.00 \\
        \checkmark & & & & & & 1,983,345,180 & 93.49 \\
        \checkmark & \checkmark & & & & & 1,891,725,045 & 89.17 \\
        \checkmark & & \checkmark & & & & 1,709,522,548 & 80.58 \\
        \checkmark & & & \checkmark & & & 1,143,660,096 & 53.91 \\
        \checkmark & & & & \checkmark & & 691,535,901 & 32.60 \\
        \checkmark & & \checkmark & & \checkmark & & 642,162,957 & 30.27 \\
        \checkmark & & & \checkmark & \checkmark & & 487,493,190 & 22.98 \\
        \checkmark & & \checkmark & & \checkmark & \checkmark & 438,358,791 & 20.66 \\
        \bottomrule
    \end{tabular}
}
\end{table}

\paragraph{Effect of top-k predicted objects and attributes.}
In Table~\ref{tab:topk}, we show that our concept distillation approach is quite robust to the choice of the number of predicted \objects and attributes. 
For $k=10$ strong accuracy is achieved with a small increase in dataset memory.

\begin{table}[t]
\centering
\caption{
    Evaluating effect of using different number of top-$k$ predicted \objects and attributes.
    Evaluation on ViT-B/16 model architecture trained for 8B processed samples on LAION-CAT.
    Memory denotes storage needed to store predicted concepts.
}
\vspace{-1em}
\label{tab:topk}
\def\arraystretch{1.0}  
\setlength{\tabcolsep}{1mm}  
\def\cww{0.7cm}
\def\hs{1mm}
\centering
\small{
    \begin{tabular}{C{1.2cm}C{1.5cm}C{\cww}C{\cww}C{\cww}C{\cww}C{\cww}}
        \toprule
        \multirow{2}{*}{top-k} & \multirow{2}{*}{Memory}  & \multirow{2}{*}{IN} & \multicolumn{2}{c}{COCO} & \multicolumn{2}{c}{Flickr} \\
        \cmidrule(l{\hs}r{\hs}){4-5}\cmidrule(l{\hs}r{\hs}){6-7}
        & & & T2I & I2T & T2I & I2T \\
        \midrule
        5  & 16.3GB & 71.4 & 42.9 & 59.4 & 72.2 & 86.5 \\
        10 & 32.6GB & 71.9 & 42.9 & 60.3 & 73.3 & 87.0 \\
        25 & 81.6GB & 71.4 & 43.1 & 60.0 & 72.9 & 87.9 \\
        \bottomrule
    \end{tabular}
}
\end{table}

\paragraph{Effect of $\alpha$ and $\beta$ on {\sc HN-Nce.}}
From intuition, one can see that the term $\alpha$ controls the mass of the positive alignment term in the loss function, and the term $\beta$ controls the difficulty of the negatives. The need for the term $\alpha$ can be attributed as follows. If there are false negatives within the dataset, dampening the positive alignment term can prevent the model from becoming overly discriminative with the true and false positive pairs. Hence, we would like to reduce $\alpha$ as the likelihood of having false positives increases (\eg, smaller datasets, less noisy training). The need for $\beta$ is straightforward: higher $\beta$ pushes the weighing function to be ``sharper'', with more mass on the hardest negatives.
Table~\ref{tab:ablations_hn} shows the effect of different values of $\alpha$ and $\beta$ on LAION-CAT.

\begin{table}[t]
\centering
\caption{
    Evaluating effect of different hyperparameters $\alpha$ and $\beta$ for the {\sc HN-NCE} loss. 
    Evaluation on ViT-B/16 model architecture trained for 16B processed samples on LAION-CAT.
}
\vspace{-1em}
\label{tab:ablations_hn}
\def\arraystretch{1.0}  
\setlength{\tabcolsep}{1mm}  
\def\cww{0.7cm}
\def\hs{1mm}
\centering
\small{
    \begin{tabular}{C{1.2cm}C{1.5cm}C{\cww}C{\cww}C{\cww}C{\cww}C{\cww}}
        \toprule
        \multirow{2}{*}{$\alpha$} & \multirow{2}{*}{$\beta$}  & \multirow{2}{*}{IN} & \multicolumn{2}{c}{COCO} & \multicolumn{2}{c}{Flickr} \\
        \cmidrule(l{\hs}r{\hs}){4-5}\cmidrule(l{\hs}r{\hs}){6-7}
        & & & T2I & I2T & T2I & I2T \\
        \midrule
        1  & 0 & 68.7 & 42.8 & 60.5 & 72.8 & 87.6 \\
        1  & 0.25 & 69.2 & 42.9 & 61.2 & 72.6 & 87.8 \\
        1  & 0.5 & 66.5 & 40.3 & 59.7 & 71.4 & 84.9 \\
        0.999 & 0.25 & 69.0 & 42.6 & 60.9 & 72.3 & 87.9  \\
        0.9 & 0.25 & 68.6 & 42.1 & 59.2 & 71.2 & 85.5  \\
        \bottomrule
    \end{tabular}
    \vspace{0.5em}
}
\end{table}

\paragraph{Additional results on few-shot probing.}
We examine the performance of our models on linear probing with the full training set for ImageNet1K~\cite{rds+15}. We compare the performance of DiHT-L/14 and CLIP-L/14~\cite{rkh+21} architectures for both the 224px and 336px input sizes in Table~\ref{tab:linear_probe}. We observe that the PGD approach with the DiHT model outperforms prior work, and also find that there is no notable difference in performance between SGD-trained and PGD-trained models, as there is no need for regularization when training with the full dataset. We reproduce the reported numbers for CLIP~\cite{rkh+21} and train our models with a learning rate of 24, no weight decay, and batch size of 96,000 for 160 epochs.

\begin{table}[t]
\centering
\caption{
    Evaluating linear probing with the complete training set for ImageNet1K on the ViT-L/14 architecture.
}
\vspace{-1em}
\label{tab:linear_probe}
\def\arraystretch{1.0}  
\setlength{\tabcolsep}{1mm}  
\def\cww{2.5cm}
\def\hs{1mm}
\centering
\small{
    \begin{tabular}{C{3cm}C{1.5cm}C{\cww}}
        \toprule
        Model & Optimizer & ImageNet-1K Accuracy (\%) \\
        \midrule
        CLIP-L/14 @ 224px & SGD & 83.60 \\
        DiHT-L/14 @ 224px & SGD & 85.40 \\
        DiHT-L/14 @ 224px & PGD & {\bf85.41} \\
        \midrule
        CLIP-L/14 @ 336px & SGD & 85.40 \\
        DiHT-L/14 @ 336px & SGD & 85.87 \\
        DiHT-L/14 @ 336px & PGD & {\bf85.89} \\
        \bottomrule
    \end{tabular}
}
\end{table}

\begin{table*}[t]
\centering
\caption{
    Zero-shot state-of-the-art dual-encoder models comparison.
    We evaluate CLIP~\cite{rkh+21} and OpenCLIP~\cite{iww+21} using our codebase.
}
\vspace{-1em}
\label{tab:sota29}
\def\arraystretch{1.0}  
\setlength{\tabcolsep}{0.1mm}  
\def\cw{0.5cm}
\def\hs{1mm}
\newcommand\cs[1]{#1}
\newcommand\css[1]{\scriptsize{#1}}
\newcommand\csss[1]{\tiny{#1}}
\centering
\scriptsize{
    \begin{tabular}{L{2cm}C{\cw}C{\cw}C{\cw}C{\cw}C{\cw}C{\cw}C{\cw}C{\cw}C{\cw}C{\cw}C{\cw}C{\cw}C{\cw}C{\cw}C{\cw}C{\cw}C{\cw}C{\cw}C{\cw}C{\cw}C{\cw}C{\cw}C{\cw}C{\cw}C{\cw}C{\cw}C{\cw}C{\cw}C{\cw}}
        \toprule
        Method & \rotatebox[origin=l]{90}{Birdsnap} & \rotatebox[origin=l]{90}{CIFAR10} & \rotatebox[origin=l]{90}{CIFAR100} & \rotatebox[origin=l]{90}{Caltech101} & \rotatebox[origin=l]{90}{Country211} & \rotatebox[origin=l]{90}{DTD} & \rotatebox[origin=l]{90}{Flowers102} & \rotatebox[origin=l]{90}{Food101} & \rotatebox[origin=l]{90}{ImageNet1K} & \rotatebox[origin=l]{90}{OxfordPets} & \rotatebox[origin=l]{90}{STL10} & \rotatebox[origin=l]{90}{SUN397} & \rotatebox[origin=l]{90}{StanfordCars} & \rotatebox[origin=l]{90}{UCF101} & \rotatebox[origin=l]{90}{HatefulMemes} & \rotatebox[origin=l]{90}{PascalVOC} & \rotatebox[origin=l]{90}{OpenImages} & \rotatebox[origin=l]{90}{COCO T2I} & \rotatebox[origin=l]{90}{COCO I2T} & \rotatebox[origin=l]{90}{Flickr T2I} & \rotatebox[origin=l]{90}{Flickr I2T} & \rotatebox[origin=l]{90}{LN-COCO T2I} & \rotatebox[origin=l]{90}{LN-COCO I2T} & \rotatebox[origin=l]{90}{LN-Flickr T2I} & \rotatebox[origin=l]{90}{LN-Flickr I2T} & \rotatebox[origin=l]{90}{Winoground T2I} & \rotatebox[origin=l]{90}{Winoground I2T} & \rotatebox[origin=l]{90}{SNLI-VE} & \rotatebox[origin=l]{90}{VQAv2} \\
        \midrule
        \multicolumn{8}{l}{~~ViT-B/32 @ 224} \\
        \cmidrule(l{\hs}r{\hs}){1-1}
        CLIP     & 40.3 & 89.8 & 65.1 & 83.9 & 17.2 & 43.8 & 66.6 & 83.9 & 63.4 & 87.4 & 97.2 & 62.3 & 59.7 & 64.2 & 58.1 & 84.2 & 27.8 & 31.4 & 49.0 & 59.5 & 79.9 & 16.8 & 24.6 & 30.2 & 38.1 & 28.1 & 27.4 & 77.6 & 57.3\\
        OpenCLIP & 50.5 & 93.6 & 75.8 & 86.4 & 16.7 & 56.1 & 71.7 & 82.7 & 66.6 & 90.6 & 96.6 & 68.5 & 86.0 & 66.1 & 53.4 & 85.4 & 34.6 & 39.0 & 56.7 & 65.7 & 81.7 & 29.5 & 35.1 & 44.0 & 51.4 & 32.0 & 30.2 & 78.6 & 59.3\\
        DiHT     & 46.5 & 92.0 & 73.6 & 80.4 & 16.3 & 55.3 & 69.8 & 84.1 & 68.0 & 91.7 & 97.2 & 66.5 & 79.6 & 68.3 & 53.5 & 78.9 & 32.4 & 40.6 & 59.3 & 68.6 & 84.4 & 29.8 & 35.7 & 46.1 & 54.0 & 30.9 & 33.0 & 79.1 & 59.9\\
        \midrule
        \multicolumn{8}{l}{~~ViT-B/16 @ 224} \\
        \cmidrule(l{\hs}r{\hs}){1-1}
        CLIP     & 43.2 & 90.8 & 68.3 & 84.7 & 22.8 & 44.9 & 71.2 & 88.7 & 68.4 & 89.1 & 98.3 & 64.4 & 64.7 & 69.5 & 59.3 & 85.3 & 29.3 & 33.7 & 51.3 & 63.3 & 81.9 & 18.7 & 25.2 & 31.3 & 37.4 & 31.0 & 30.2 & 77.9 & 57.7\\
        OpenCLIP & 52.1 & 91.7 & 71.4 & 86.2 & 18.1 & 50.8 & 69.3 & 86.1 & 67.1 & 89.4 & 97.0 & 69.6 & 83.8 & 67.7 & 55.7 & 84.2 & 35.2 & 37.8 & 55.4 & 65.2 & 84.1 & 26.1 & 33.1 & 43.5 & 46.9 & 30.5 & 30.2 & 78.4 & 59.3\\
        DiHT     & 54.5 & 92.7 & 77.5 & 81.2 & 19.1 & 59.4 & 70.5 & 89.1 & 72.2 & 92.7 & 98.2 & 68.4 & 86.0 & 70.3 & 56.2 & 79.5 & 34.6 & 43.3 & 60.3 & 72.9 & 89.8 & 32.4 & 38.2 & 52.9 & 57.7 & 32.0 & 33.4 & 80.8 & 60.3\\
        \midrule
        \multicolumn{8}{l}{~~ViT-L/14 @ 224} \\
        \cmidrule(l{\hs}r{\hs}){1-1}
        CLIP     & 52.5 & 95.6 & 78.2 & 86.7 & 31.9 & 55.5 & 79.1 & 93.1 & 75.6 & 93.5 & 99.4 & 67.6 & 77.8 & 77.0 & 60.4 & 85.5 & 30.6 & 36.5 & 54.9 & 66.1 & 84.5 & 20.8 & 28.6 & 36.2 & 44.2 & 31.9 & 32.0 & 78.2 & 58.4\\
        OpenCLIP & 62.9 & 96.6 & 83.4 & 88.0 & 26.3 & 62.9 & 75.5 & 91.0 & 75.2 & 93.2 & 98.9 & 74.3 & 92.6 & 75.2 & 55.1 & 87.5 & 38.0 & 46.2 & 64.3 & 75.4 & 90.4 & 34.6 & 39.9 & 50.9 & 57.7 & 33.4 & 36.4 & 80.8 & 60.0\\
        DiHT     & 60.4 & 91.7 & 81.3 & 81.6 & 26.0 & 60.3 & 77.6 & 92.7 & 77.0 & 93.8 & 98.0 & 70.2 & 91.1 & 77.9 & 56.5 & 79.3 & 35.0 & 48.0 & 65.1 & 76.7 & 92.0 & 35.6 & 40.7 & 52.7 & 60.3 & 31.8 & 33.4 & 81.3 & 61.0\\
        \midrule
        \multicolumn{8}{l}{~~ViT-L/14 @ 336} \\
        \cmidrule(l{\hs}r{\hs}){1-1}
        CLIP & 53.7 & 95.0 & 77.0 & 87.2 & 34.4 & 56.0 & 78.6 & 93.8 & 76.6 & 93.8 & 99.5 & 68.7 & 79.2 & 77.6 & 61.6 & 86.2 & 31.8 & 37.7 & 57.1 & 68.6 & 86.6 & 20.2 & 28.6 & 38.1 & 45.7 & 32.3 & 21.4 & 78.7 & 58.5\\
        DiHT & 62.0 & 92.2 & 81.2 & 82.4 & 27.8 & 61.1 & 77.0 & 92.9 & 77.9 & 94.0 & 98.2 & 71.2 & 91.5 & 77.7 & 56.3 & 81.0 & 36.5 & 49.3 & 65.3 & 78.2 & 91.1 & 36.7 & 41.2 & 54.5 & 61.6 & 35.0 & 38.5 & 81.7 & 61.4\\
        \bottomrule
    \end{tabular}
}
\end{table*}

\paragraph{Additional results on zero-shot benchmark.}
We report performance of CLIP~\cite{rkh+21}, OpenCLIP~\cite{iww+21}, and \ours on all 29 zero-shot tasks in Table~\ref{tab:sota29}.

\subsection{Contrastive Alignment with Hard Negatives}
\paragraph{Convergence guarantees}
\begin{proposition}
Let $\cL^\star(\bphi_i, \bphi_t) = \sup_{q\in\Pi} \cL_(\bphi_i, \bphi_t, q)$. Then for any measurable $\bphi_i, \bphi_t : \cX \rightarrow \bbS^{d-1}$ and $\tau = \cO(1)$ we observe the convergence $\cL_(\bphi_i, \bphi_t, q) \rightarrow \cL^\star(\bphi_i, \bphi_t)$ as $\beta \rightarrow \infty$.
\end{proposition}
\begin{proof}
Follows from Proposition 6 of~\cite{rcsj21} with the loss function $\cL(\bphi_i, \bphi_t, q_\beta)$ defined as follows for any $\beta$.
\begin{align*}
&\cL(\bphi_i, \bphi_t, q_\beta) = \\
    &\log\left[\frac{e^{\bphi_i(x)^\top\bphi_t(x)/\tau}}{ e^{\bphi_i(x)^\top\bphi_t(x)/\tau}+Q\cdot\bbE_{y\sim q_\beta}\left[ e^{\bphi_i(x)^\top\bphi_t(y)/\tau}\right]}\right] \\
    &+ \log\left[\frac{e^{\bphi_i(x)^\top\bphi_t(x)/\tau}}{ e^{\bphi_i(x)^\top\bphi_t(x)/\tau}+Q\cdot\bbE_{y\sim q_\beta}\left[ e^{\bphi_i(x)^\top\bphi_t(y)/\tau}\right]}\right].
\end{align*}
\end{proof}

\end{document}